\documentclass[twoside,11pt]{article}
\usepackage[abbrvbib]{jmlr2e_arxiv}

\usepackage{subcaption}

\usepackage{amsmath}

\newcommand\st{~~~\text{s.t.}~~~}

\def\eg{\emph{e.g.}}
\def\ie{\emph{i.e.}}
\def\Hcal{{\mathcal H}}
\def\Kcal{{\mathcal K}}
\def\Fcal{{\mathcal F}}
\def\Pcal{{\mathcal P}}
\def\Xcal{{\mathcal X}}
\def\Real{{\mathbb R}}
\def\vs{}
\def\kmone{k\text{--}1}
\def\kmtwo{k\text{--}2}
\def\nmone{n\text{--}1}
\def\defin{:=}

\def\R{{\mathbb R}}
\def\N{{\mathbb N}}
\def\Z{{\mathbb Z}}
\def\Hc{{\mathcal H}}
\def\Hcbar{\bar{\mathcal H}}
\def\activ{\sigma}

\DeclareMathOperator{\vect}{span}

\DeclareMathOperator{\1}{\mathbf{1}}

\newtheorem{appxtheorem}{Theorem}[section]
\newtheorem{appxlemma}{Lemma}[section]

\let\Omm\Omega
\renewcommand\Omega{{\R^d}}

\title{Group Invariance, Stability to Deformations,\\ and Complexity of Deep Convolutional
Representations}

\author{\name Alberto Bietti \email alberto.bietti@inria.fr \\
		\name Julien Mairal \email julien.mairal@inria.fr \\
		\addr Univ. Grenoble Alpes, Inria, CNRS, Grenoble INP$\hspace*{0.03cm}^*$\hspace*{-0.10cm}, LJK, 38000 Grenoble, France
}
\jmlrheading{1}{2018}{1-49}{4/00}{10/00}{bietti}{Alberto Bietti and Julien Mairal}
\ShortHeadings{Invariance, Stability, and Complexity of Deep Convolutional Representations}{Bietti and Mairal}
\firstpageno{1}

\begin{document}

\maketitle

\begin{abstract}%

The success of deep convolutional architectures is often attributed in part to their
ability to learn multiscale and invariant representations of natural signals.
However, a precise study of these properties and how they affect learning guarantees
is still missing.
In this paper, we consider deep convolutional representations of signals; we study
their invariance to translations and to more general groups of transformations,
their stability to the action of diffeomorphisms, and their ability to preserve signal information.
This analysis is carried 
by introducing a multilayer kernel based on convolutional kernel networks
and by studying the geometry induced by the kernel mapping.
We then characterize the corresponding reproducing kernel Hilbert space (RKHS),
showing that it contains a large class of convolutional neural networks
with homogeneous activation functions.
This analysis allows us to separate data representation from learning,
and to provide a canonical measure of model complexity, the RKHS norm,
which controls both stability and generalization of any learned model.
In addition to models in the constructed RKHS, our stability analysis also applies to convolutional networks
with generic activations such as rectified linear units,
and we discuss its relationship with recent generalization bounds based on spectral~norms.

\end{abstract}

\def\thefootnote{\fnsymbol{footnote}}
\footnotetext[1]{Institute of Engineering Univ. Grenoble Alpes}
\def\thefootnote{\arabic{footnote}}

\section{Introduction} 
\label{sec:introduction}
The results achieved by deep neural networks for prediction tasks have been
impressive in domains where data is structured and available in large amounts.
In particular, convolutional neural networks~\citep[CNNs,][]{lecun1989backpropagation}
have shown to effectively leverage the local stationarity of natural images at
multiple scales thanks to convolutional operations,
while also providing some translation invariance through
pooling operations. Yet, the exact nature of this invariance and the
characteristics of functional spaces where convolutional neural networks live
are poorly understood; overall, these models are sometimes seen as
clever engineering black boxes that have been designed with a lot of insight collected
since they were introduced.

Understanding the inductive bias of these models is nevertheless
a fundamental question. For instance, a better grasp of the geometry
induced by convolutional representations may bring new intuition about their success,
and lead to improved measures of model complexity.
In turn, the issue of regularization may be solved by providing ways to control the variations of prediction
functions in a principled manner.
One meaningful way to study such variations is to consider the stability of model predictions
to naturally occuring changes of input signals, such as translations and deformations.

Small deformations of natural signals often preserve their main characteristics, such as
class labels (\eg, the same digit with different handwritings
may correspond to the same images up to small deformations), and provide a much richer class
of transformations than translations.
The scattering transform~\citep{mallat2012group,bruna2013invariant} is a recent attempt to
characterize convolutional multilayer architectures based on 
wavelets.  The theory provides an elegant characterization of invariance and stability
properties of signals represented via the scattering operator,
through a notion of Lipschitz stability to the action of diffeomorphisms.
Nevertheless, these networks do not involve ``learning'' in the classical sense since the
filters of the networks are pre-defined, and the resulting architecture differs
significantly from the most used ones, which adapt filters to training data.

In this work, we study these theoretical properties for more standard
convolutional architectures, from the point of view of positive definite
kernels~\citep{scholkopf2001learning}. Specifically, we consider a functional
space derived from a kernel for multi-dimensional signals
that admits a multi-layer and convolutional structure based on the construction
of convolutional kernel networks (CKNs) introduced by \citet{mairal_end--end_2016,mairal2014convolutional}.
The kernel representation follows standard convolutional architectures,
with patch extraction, non-linear (kernel) mappings, and pooling operations.
We show that our functional space contains a large class of CNNs with
smooth homogeneous activation functions.

The main motivation for introducing a kernel framework is to study separately
data representation and predictive models. On the one hand, we study
the translation-invariance properties of the kernel representation and its
stability to the action of diffeomorphisms, obtaining similar guarantees as the
scattering transform~\citep{mallat2012group}, while preserving signal
information.
When the kernel is appropriately designed, we also show how to
obtain signal representations that are invariant to the action of any locally compact
group of transformations, by modifying the construction of the kernel representation
to become \emph{equivariant} to the group action.
On the other hand, we show that these stability
results can be translated to predictive models by
controlling their norm in the functional space, or simply the norm of the
last layer in the case of CKNs~\citep{mairal_end--end_2016}.
With our kernel framework, the RKHS norm also acts as a measure of model complexity,
thus controlling both stability and generalization,
so that stability may lead to improved sample complexity.
Finally, our work suggests that explicitly regularizing CNNs with the RKHS norm
(or approximations thereof) can help obtain more stable models,
a more practical question which we study in follow-up work~\citep{bietti2018regularization}.

A short version of this paper was published at the Neural Information Processing Systems 2017 conference~\citep{bietti2017invariance}.

\subsection{Summary of Main Results}
Our work characterizes properties of deep convolutional models along two main directions.
\begin{itemize}
	\item The first goal is to study \emph{representation} properties of such models,
	independently of training data.
	Given a deep convolutional architecture,
	we study signal preservation as well as invariance and stability properties.
	\item The second goal focuses on \emph{learning} aspects, by studying the complexity
	of learned models based on our representation.
	In particular, our construction relies on kernel methods, allowing us to define a corresponding functional
	space (the RKHS).
	We show that this functional space contains a class of CNNs with smooth homogeneous activations,
	and study the complexity of such models by considering their RKHS norm.
	This directly leads to statements on the generalization of such models, as well
	as on the invariance and stability properties of their predictions.
	\item Finally, we show how some of our arguments extend to more traditional CNNs with
	generic and possibly non-smooth activations (such as ReLU or tanh).
\end{itemize}

\paragraph{Signal preservation, invariance and stability.}
We tackle this first goal by defining a deep convolutional representation based on
hierarchical kernels. We show that the representation preserves signal information and guarantees
near-invariance to translations and stability to deformations in the following sense,
defined by~\citet{mallat2012group}:
for signals $x : \Omega \to \R^{p_0}$ defined on the continuous domain $\R^d$,
we say that a representation $\Phi(x)$ is \emph{stable} to the action of diffeomorphisms if
\begin{equation*}
\|\Phi(L_\tau x) - \Phi(x) \| \leq (C_1 \|\nabla \tau\|_\infty + C_2 \|\tau\|_\infty)\|x\|,
\end{equation*}
where $\tau:\Omega \to \Omega$ is a $C^1$-diffeomorphism, $L_\tau x(u) = x(u - \tau(u))$ its action operator,
and the norms $\|\tau\|_\infty$ and $\|\nabla \tau\|_\infty$ characterize how large the
translation and deformation components are, respectively (see Section~\ref{sec:stability} for formal definitions).
The Jacobian $\nabla \tau$ quantifies the size of local deformations,
so that the first term controls the stability of the representation.
In the case of translations, the first term vanishes ($\nabla \tau = 0$),
hence a small value of $C_2$ is desirable for translation invariance.
We show that such signal preservation and stability properties are valid for the
multilayer kernel representation~$\Phi$ defined in Section~\ref{sec:kernel_construction}
by repeated application of patch extraction, kernel mapping, and pooling operators:
\begin{itemize}
	\item The representation can be discretized with no loss of information,
	by subsampling at each layer with a factor smaller than the patch size;
	\item The translation invariance is controlled by a factor $C_2 = C_2' / \sigma_n$,
	where~$\sigma_n$ represents the ``resolution'' of the last layer,
	and typically increases exponentially with depth;
	\item The deformation stability is controlled by a factor $C_1$ which increases as~$\kappa^{d+1}$,
	where~$\kappa$ corresponds to the patch size at a given layer, that is,
	the size of the ``receptive field'' of a patch relative to the resolution of the previous layer.
\end{itemize}
These results suggest that a good way to obtain a stable representation that preserves signal information
is to use the smallest possible patches at each layer (\eg, 3x3 for images) and perform pooling and downsampling
at a factor smaller than the patch size, with as many layers as needed in order to reach a desired level
of translation invariance~$\sigma_n$.
We show in Section~\ref{sub:stability_with_approximation} that the same invariance and stability guarantees
hold when using kernel approximations as in CKNs,
at the cost of losing signal information.

In Section~\ref{sec:global_invariance_to_group_actions}, we show how to go beyond the translation group,
by constructing similar representations that are invariant to the action of locally compact groups.
This is achieved by modifying patch extraction and pooling operators so that they commute with
the group action operator (this is known as \emph{equivariance}).

\paragraph{Model complexity.}
Our second goal is to analyze the complexity of deep convolutional models
by studying the functional space defined by our kernel representation,
showing that certain classes of CNNs are contained in this space,
and characterizing their norm.

The multi-layer kernel representation defined in Section~\ref{sec:kernel_construction} is
constructed by using kernel mappings defined on local signal patches at each scale,
which replace the linear mapping followed by a non-linearity in standard convolutional networks.
Inspired by~\citet{zhang2016convexified}, we show in Section~\ref{sub:rkhs_activations} that when these kernel mappings come from a class of dot-product kernels, the corresponding RKHS contains functions of the form
\begin{equation*}
  z \mapsto \|z\| \sigma(\langle g, z \rangle / \|z\|),
\end{equation*}
for certain types of smooth activation functions~$\sigma$, where $g$ and $z$ live in a particular Hilbert space.
These behave like simple neural network functions on patches, up to homogeneization.
Note that if~$\sigma$ was allowed to be homogeneous, such as for rectified linear units $\sigma(\alpha) = \max(\alpha, 0)$,
homogeneization would disappear.
By considering multiple such functions at each layer, we construct a CNN in the RKHS of the full multi-layer
kernel in Section~\ref{sub:cnns_rkhs}.
Denoting such a CNN by $f_\sigma$, we show that its RKHS norm can be bounded as
\begin{equation*}
\|f_{\sigma}\|^2 \leq \|w_{n+1}\|^2 ~ C_\sigma^2(\|W_n\|_2^2 ~ C_\sigma^2(\|W_{n-1}\|_2^2 \ldots C_\sigma^2(\|W_2\|_2^2~C_\sigma^2(\|W_1\|_F^2)) \ldots)),
\end{equation*}
where $W_k$ are convolutional filter parameters at layer~$k$,
$w_{n+1}$ carries the parameters of a final linear fully connected layer,
$C_\sigma^2$ is a function quantifying the complexity of the simple functions defined above
depending on the choice of activation~$\sigma$,
and~$\|W_k\|_2$, $\|W_k\|_F$ denote spectral and Frobenius norms, respectively,
(see Section~\ref{sub:cnns_rkhs} for details).
This norm can then control generalization aspects through classical margin bounds,
as well as the invariance and stability of model predictions. Indeed, by using the reproducing property
$f(x) = \langle f, \Phi(x) \rangle$,
this ``linearization'' lets us control stability properties of model predictions through $\|f\|$: 
\[
\text{for all signals}~x~\text{and}~x',~~~|f(x) - f(x')| \leq \|f\| \cdot \|\Phi(x) - \Phi(x')\|,
\]
meaning that the prediction function~$f$ will inherit the stability of $\Phi$ when $\|f\|$ is small.

\paragraph{The case of standard CNNs with generic activations.}
When considering CNNs with generic, possibly non-smooth activations such as rectified linear units (ReLUs),
the separation between a data-independent representation and a learned model is not always
achievable in contrast to our kernel approach.
In particular, the ``representation'' given by the last layer of a learned CNN is often considered by practitioners,
but such a representation is data-dependent in that it is typically trained on a specific task and dataset,
and does not preserve signal information.

Nevertheless, we obtain similar invariance and stability properties for
the predictions of such models in Section~\ref{sub:generic_activations},
by considering a complexity measure given by the product of spectral norms of each linear convolutional mapping in a CNN.
Unlike our study based on kernel methods, such results do not say anything about generalization;
however, relevant generalization bounds based on similar quantities have been derived
(though other quantities in addition to the product of spectral norms appear in the bounds,
and these bounds do not directly apply to CNNs), \eg, by~\citet{bartlett2017spectrally,neyshabur2017pac},
making the relationship between generalization and stability clear in this context as well.

\subsection{Related Work}
Our work relies on image representations introduced in the context of
convolutional kernel
networks~\citep{mairal_end--end_2016,mairal2014convolutional}, which yield a
sequence of spatial maps similar to traditional CNNs, but where
each point on the maps is possibly infinite-dimensional and lives in a reproducing kernel Hilbert space (RKHS).
The extension to signals with $d$ spatial dimensions is straightforward.
Since computing the corresponding Gram matrix as in
classical kernel machines is computationally impractical, 
CKNs provide an approximation scheme
consisting of learning finite-dimensional subspaces of each RKHS's layer, where
the data is projected. The resulting architecture of CKNs resembles traditional
CNNs with a subspace learning interpretation and different unsupervised learning~principles.

Another major source of inspiration is the study of group-invariance and
stability to the action of diffeomorphisms of scattering
networks~\citep{mallat2012group}, which introduced the main formalism and
several proof techniques that were keys to our results.
Our main effort was to extend them to more general CNN architectures and to the kernel framework,
allowing us to provide a clear relationship between stability properties of the representation
and generalization of learned CNN models.
We note that an extension of scattering networks results to more general
convolutional networks was previously given by~\citet{wiatowski2018mathematical};
however, their guarantees on deformations do not improve on the inherent stability properties
of the considered signal, and their study does not consider learning or generalization,
by treating a convolutional architecture with fixed weights as a feature extractor.
In contrast, our stability analysis shows the benefits of deep representations with a clear dependence
on the choice of network architecture through the size of convolutional patches and pooling layers,
and we study the implications for learned CNNs through notions of model complexity.

Invariance to groups of transformations was also studied for more
classical convolutional neural networks from methodological and empirical
points of view~\citep{bruna2013learning,cohen2016group}, and for shallow
learned representations~\citep{anselmi2016invariance} or kernel methods~\citep{haasdonk2007invariant,mroueh2015learning,raj2017local}.
Our work provides a similar group-equivariant construction to~\citep{cohen2016group}, while additionally
relating it to stability. In particular, we show that in order to achieve group invariance,
pooling on the group is only needed at the final layer,
while deep architectures with pooling at multiple scales are mainly beneficial for stability.
For the specific example of the roto-translation group~\citep{sifre2013rotation}, we show that our construction
achieves invariance to rotations while maintaining stability to deformations on the translation group.

Note also that other techniques combining deep neural networks and kernels have been
introduced earlier. Multilayer kernel machines were for instance introduced by~\citet{cho2009kernel,scholkopf1998nonlinear}.
Shallow kernels for images modeling local regions were also proposed by~\citet{scholkopf1997}, and a multilayer construction was proposed by~\citet{bo2011}.
More recently, different models based on kernels have been introduced
by~\citet{anselmi2015deep,daniely2016,montavon2011} to gain some
theoretical insight about classical multilayer neural networks, while kernels are used by~\citet{zhang2016convexified} to
define convex models for two-layer convolutional networks.
Theoretical and practical concerns for learning with multilayer kernels
have been studied in~\citet{daniely2017random,daniely2016,steinwart2016learning,zhang2016l1}
in addition to CKNs.
In particular, \citet{daniely2017random,daniely2016} study certain classes of dot-product kernels
with random feature approximations, \citet{steinwart2016learning} consider
hierarchical Gaussian kernels with learned weights, and~\citet{zhang2016l1} study a convex
formulation for learning a certain class of fully connected neural networks using a hierarchical kernel.
In contrast to these works, our focus is on the kernel \emph{representation} induced by the specific
hierarchical kernel defined in CKNs and the geometry of the RKHS.
Our characterization of CNNs and activation functions contained in the RKHS is similar to
the work of~\citet{zhang2016l1,zhang2016convexified}, but differs in several ways:
we consider general \emph{homogeneous} dot-product kernels, which yield desirable properties of kernel mappings for stability;
we construct generic multi-layer CNNs with pooling in the RKHS, while~\citet{zhang2016l1} only considers fully-connected networks
and~\citet{zhang2016convexified} is limited to two-layer convolutional networks with no pooling;
we quantify the RKHS norm of a CNN depending on its parameters, in particular matrix norms,
as a way to control stability and generalization,
while \citet{zhang2016l1,zhang2016convexified} consider models with constrained parameters,
and focus on convex learning procedures.

\subsection{Notation and Basic Mathematical Tools}
A positive definite kernel~$K$ that operates on a set
$\Xcal$ implicitly defines a reproducing kernel Hilbert space $\Hcal$ of
functions from~$\Xcal$ to~$\Real$, along with a mapping~$\varphi: \Xcal \to
\Hcal$. A \emph{predictive model} associates to every point~$z$
in~$\Xcal$ a label in~$\Real$; it consists of a linear function~$f$ in~$\Hcal$ such
that $f(z) = \langle f, \varphi(z) \rangle_{\Hcal}$, where
$\varphi(z)$ is the \emph{data representation}.
Given now two points $z, z'$ in $\Xcal$, Cauchy-Schwarz's inequality allows us to 
control the variation of the predictive model~$f$ according to the geometry induced by the Hilbert norm $\|.\|_{\Hcal}$:
\begin{equation}
   |f(z) - f(z')| \leq \|f\|_{\Hc} \|\varphi(z) - \varphi(z')\|_{\Hc}. \label{eq:cs}
\end{equation}
This property implies that two points $z$ and~$z'$ that are close to each other
according to the RKHS norm should lead to similar predictions, when the model
$f$ has small norm in~$\Hcal$. 

Then, we consider notation from signal processing similar
to~\citet{mallat2012group}. We call
a signal~$x$ a function in $L^2(\Omega, \Hc)$, where the domain~$\Omega$
represents spatial coordinates, and $\Hc$ is a Hilbert space,
when $\|x\|_{L^2}^2 := \int_\Omega \|x(u)\|^2_{\Hc} du < \infty$, where~$du$ is the
Lebesgue measure on $\R^d$.
Given a linear operator~$T:L^2(\Omega, \Hc) \to
L^2(\Omega, \Hc')$, the operator norm is defined as $\|T\|_{L^2(\Omega, \Hc)
\to L^2(\Omega, \Hc')} := \sup_{\|x\|_{L^2(\Omega, \Hc)} \leq 1}
\|Tx\|_{L^2(\Omega, \Hc')}$. For the sake of clarity, we drop norm
subscripts, from now on, using the notation $\|\cdot\|$ for Hilbert space
norms, $L^2$ norms, and $L^2 \to L^2$ operator norms, while $|\cdot|$ 
denotes the Euclidean norm on~$\R^d$.
We use cursive capital letters (\eg, $\Hcal, \Pcal$) to denote Hilbert spaces,
and non-cursive ones for operators (\eg, $P, M, A$).
Some useful mathematical tools are also
presented in Appendix~\ref{sec:basic_tools}.

\subsection{Organization of the Paper}
The rest of the paper is structured as follows:
\begin{itemize}
	\item In Section~\ref{sec:kernel_construction}, we introduce 
	a multilayer convolutional kernel representation for continuous signals,
	based on a hierarchy of patch extraction, kernel mapping,
	and pooling operators. We present useful properties of this representation such as signal preservation,
	as well as ways to make it practical through discretization and kernel approximations in the context of~CKNs.
	\item In Section~\ref{sec:stability}, we present our main results regarding stability and invariance,
	namely that the kernel representation introduced in Section~\ref{sec:kernel_construction} is near translation-invariant
	and stable to the action of diffeomorphisms. We then show in Section~\ref{sub:stability_with_approximation} that
	the same stability results apply in the presence of kernel approximations such as those of CKNs~\citep{mairal_end--end_2016},
	and describe a generic way to modify the multilayer construction in order to guarantee invariance to the action of
	any locally compact group of transformations in Section~\ref{sec:global_invariance_to_group_actions}.
	\item In Section~\ref{sec:link_with_cnns}, we study the functional spaces induced by our representation,
	showing that simple neural-network like functions with certain smooth activations are contained in
	the RKHS at intermediate layers, and that the RKHS of the full kernel induced by our representation
	contains a class of generic CNNs with smooth and homogeneous activations. We then present upper bounds
	on the RKHS norm of such CNNs, which serves as a measure of complexity, controlling both generalization and stability.
	Section~\ref{sub:generic_activations} studies the stability for CNNs with generic activations such as rectified linear units,
	and discusses the link with generalization.
	\item Finally, we discuss in Section~\ref{sec:discussion_conclusion} how the obtained stability results
	apply to the practical setting of learning prediction functions. In particular, we explain why the
	regularization used in CKNs provides a natural way to control stability,
	while a similar control is harder to achieve with generic CNNs.
\end{itemize}


\section{Construction of the Multilayer Convolutional Kernel} 
\label{sec:kernel_construction}
We now present the  multilayer convolutional kernel, which operates on signals
with $d$ spatial dimensions. The construction follows closely that of
convolutional kernel networks but is generalized
to input signals defined on the continuous domain $\Real^d$.
Dealing with continuous signals is indeed useful to characterize the stability properties
of signal representations to small deformations, as done by~\citet{mallat2012group} in the context of the
scattering transform. The issue of discretization on a discrete
grid is addressed in Section~\ref{sub:kernel_approximation_and_discretization}. 

In what follows, we consider signals $x_0$ that live in
$L^2(\Omega,\Hcal_0)$, where typically $\Hcal_0 = \R^{p_0}$ (\eg, with
$p_0=3$ and $d=2$, the vector $x_0(u)$ in $\R^3$ may represent the RGB pixel value at location $u$ in~$\R^2$).
Then, we build a sequence
of reproducing kernel Hilbert spaces $\Hcal_1, \Hcal_2, \ldots,$ and transform $x_0$ into a
sequence of ``feature maps'', respectively denoted by~$x_1$ in
$L^2(\Omega,\Hcal_1)$, $x_2$ in $L^2(\Omega,\Hcal_2)$, \emph{etc...}  As
depicted in Figure~\ref{fig:kernel}, a new map~$x_k$ is built from the previous
one~$x_{\kmone}$ by applying successively three operators that perform patch
extraction ($P_k$), kernel mapping $(M_k)$ to a new RKHS~$\Hcal_k$, and linear pooling $(A_k)$, respectively. 
When going up in the hierarchy, the points $x_k(u)$ carry information from
larger signal neighborhoods centered at $u$ in $\Omega$ with more invariance,
as we formally show in Section~\ref{sec:stability}. 

\begin{figure}[tb]
	\begin{center}
	\includegraphics[width=0.96\textwidth]{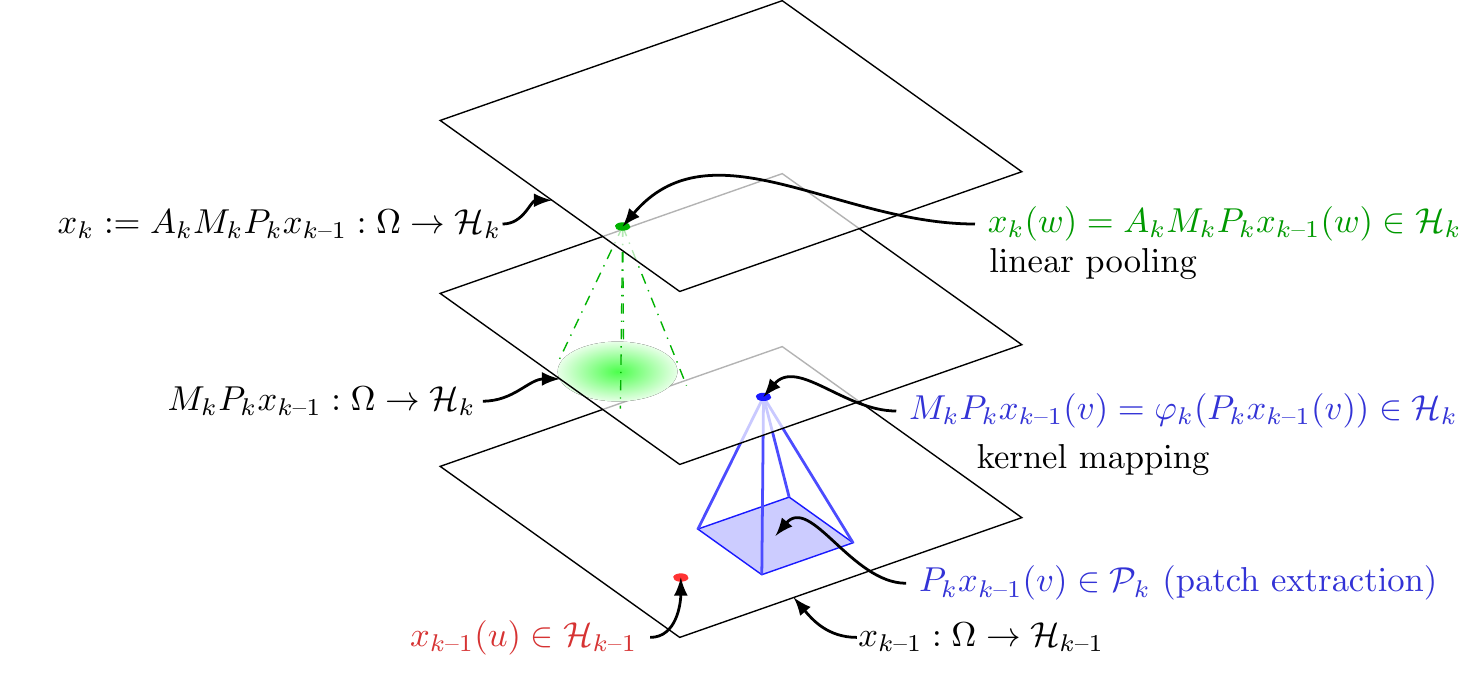}

	\end{center}
	\caption{Construction of the $k$-th signal representation from the $\kmone$-th one. Note that while the domain~$\Omm$ is depicted as a box in~$\R^2$ here, our construction is supported on $\Omm = \R^d$.}
	\label{fig:kernel}
\end{figure}

\paragraph{Patch extraction operator.}
Given the layer~$x_{\kmone}$, we consider a patch shape $S_k$, defined as a
compact centered subset of~$\Omega$, \eg,~a box, and we define the
Hilbert space $\Pcal_k \defin L^2(S_k,\Hcal_{\kmone})$ equipped with the norm
$\|z\|^2=\int_{S_k} \|z(u)\|^2 d
\nu_k(u)$, where $d \nu_k$ is the normalized uniform measure on $S_k$ for every
$z$ in~$\Pcal_k$.  Specifically, we define the (linear) patch extraction operator $P_k:
L^2(\Omega,\Hcal_{\kmone}) \to L^2(\Omega,\Pcal_k)$ such that for all $u$ in~$\Omega$,
\begin{equation*}
   P_k x_{\kmone}(u) = (v \mapsto x_{\kmone}(u + v))_{v \in S_k} \in \Pcal_k.
\end{equation*}
Note that by equipping~$\Pcal_k$ with a normalized measure, it is easy to show
that the operator~$P_k$ preserves the norm---that is, $\|P_k x_{\kmone}\| =
\|x_{\kmone}\|$ and hence~$P_k x_{\kmone}$ is in $L^2(\Omega, \Pcal_k)$.

\paragraph{Kernel mapping operator.}
Then, we map each patch of $x_{\kmone}$ to a RKHS~$\Hcal_k$ thanks to the
kernel mapping $\varphi_k: \Pcal_k \to \Hcal_k$ associated to a positive definite kernel $K_k$ that operates on patches.  It
allows us to define the pointwise operator $M_k$ such that for all $u$ in~$\Omega$,
\begin{equation*}
   M_k P_k x_{\kmone}(u) \defin \varphi_k( P_k x_{\kmone}(u)) \in \Hc_k.
\end{equation*}
In this paper, we consider homogeneous dot-product kernels~$K_k$ operating on~$\Pcal_k$,
defined in terms of a function $\kappa_k: [-1,1] \to \Real$ that satisfies the following constraints:
\begin{equation}
\tag{A1}\label{eq:kappa_assumption}
\kappa_k(u) = \sum_{j=0}^{+\infty} b_j u^j \quad \text{s.t.}  \quad \forall j, b_j \geq 0, \quad \kappa_k(1)=1, \quad \kappa_k'(1)=1,
\end{equation}
assuming convergence of the series $\sum_j b_j$ and $\sum_j j b_j$.
Then, we define the kernel~$K_k$ by
      \begin{equation}
         \label{eq:dp_kernel_appx}
         K_k(z, z') = \|z\| \|z'\| \kappa_k\!\left(\frac{\langle z, z'\rangle}{\|z\| \|z'\|} \right),
      \end{equation}
      if $z,z' \in \Pcal_k \setminus \{0\}$, and $K_k(z,z')=0$ if $z=0$ or $z' =0$.
      The kernel is positive definite since it admits a Maclaurin expansion with only non-negative
      coefficients~\citep{schoenberg,scholkopf2001learning}.
      The condition $\kappa_k(1)=1$ ensures that the RKHS mapping preserves the
      norm---that is, $\|\varphi_k(z)\|=K_k(z,z)^{1/2}=\|z\|$, and thus
      $\|M_k P_k x_{\kmone} (u)\| = \| P_k x_{\kmone} (u)\|$ for all~$u$
      in~$\Omega$; as a consequence, $M_k P_k x_{\kmone}$ is always in
      $L^2(\Omega, \Hc_k)$. 
      The technical condition $\kappa_k'(1) = 1$, where
      $\kappa_k'$ is the first derivative of $\kappa_k$, ensures that the
      kernel mapping~$\varphi_k$ is non-expansive, according to
      Lemma~\ref{lemma:dp_kernels} below.  
\begin{lemma}[Non-expansiveness of the kernel mappings]\label{lemma:dp_kernels}
   Consider a positive-definite kernel of the form~\eqref{eq:dp_kernel_appx} satisfying~\eqref{eq:kappa_assumption}
   with RKHS mapping $\varphi_k: \Pcal_k \to \Hcal_k$. Then, $\varphi_k$ is non-expansive---that is, for all $z,z'$ in $\Pcal_k$,
   \begin{displaymath}
      \|\varphi_k(z)-\varphi_k(z')\| \leq \|z-z'\|.
   \end{displaymath}
   Moreover, we remark that the kernel $K_k$ is lower-bounded by the linear one
   \begin{equation}
      K_k(z,z') \geq \langle z, z' \rangle. \label{eq:lower_linear}
   \end{equation}
\end{lemma}

    From the proof of the lemma, given in
      Appendix~\ref{sec:choices_of_kernels}, one may notice that the assumption
      $\kappa_k'(1) = 1$ is not critical and 
may be safely replaced by $\kappa_k'(1) \leq 1$. Then, the 
non-expansiveness property would be preserved. Yet, we have chosen a
stronger constraint since it yields a few simplifications in the stability analysis,
      where we use the relation~(\ref{eq:lower_linear}) that requires $\kappa_k'(1)=1$.
More generally, the kernel mapping is Lipschitz continuous with constant $\rho_k = \max(1, \sqrt{\kappa_k'(1)})$.
Our stability results hold in a setting with $\rho_k > 1$, but with constants $\prod_{k} \rho_k$ that may grow
exponentially with the number of layers.

Examples of
functions~$\kappa_k$ that satisfy the properties~\eqref{eq:kappa_assumption} are now
given below:
\vspace*{-0.4cm}
\begin{center}
   \renewcommand{\arraystretch}{1.5}
\begin{tabular}{|l|l|}
   \hline
   exponential & $\kappa_{\text{exp}}(\langle z, z' \rangle) = e^{\langle z, z' \rangle  - 1}$ \\
   \hline
   inverse polynomial & $\kappa_{\text{inv-poly}}(\langle z, z' \rangle) =  \frac{1}{2-\langle z, z' \rangle}$ \\
   \hline
   polynomial, degree $p$ & $\kappa_{\text{poly}}(\langle z, z' \rangle) = \frac{1}{(c+1)^p}(c+ \langle z, z' \rangle)^p~~~~\text{with}~~~ c= p-1$ \\ 
   \hline
   arc-cosine, degree 1 & $\kappa_{\text{acos}}(\langle z, z' \rangle) = \frac{1}{\pi}\left(\sin(\theta) + (\pi-\theta)\cos(\theta)\right) ~\text{with}~ \theta = \arccos(\langle z, z' \rangle)$ \\
   \hline
   Vovk's, degree 3 & $\kappa_{\text{vovk}}(\langle z, z' \rangle) = \frac{1}{3} \left(\frac{1-\langle z, z' \rangle^3}{1-\langle z, z' \rangle}\right) = \frac{1}{3}\left(1+ \langle z, z' \rangle + \langle z, z' \rangle^2\right)$ \\
   \hline
\end{tabular}
\end{center}

We note that the inverse
polynomial kernel was used by~\citet{zhang2016l1,zhang2016convexified} to build
convex models of fully connected networks and two-layer convolutional neural networks,
while the arc-cosine kernel appears in early deep kernel machines~\citep{cho2009kernel}.
Note that the homogeneous exponential kernel reduces to the Gaussian
kernel for unit-norm vectors. Indeed, for all $z, z'$ such that $\|z\|=\|z'\|=1$, we have 
\begin{displaymath}
   \kappa_{\text{exp}}(\langle z, z' \rangle) = e^{\langle z,z' \rangle -1} = e^{-\frac{1}{2}\|z-z'\|^2},
\end{displaymath}
and thus, we may refer to kernel~(\ref{eq:dp_kernel_appx}) with the function $\kappa_{\text{exp}}$ as the homogeneous Gaussian kernel. The kernel 
$\kappa(\langle z, z' \rangle) = e^{\alpha (\langle z,z' \rangle -1)} = e^{-\frac{\alpha}{2}\|z-z'\|^2}$ with $\alpha \neq 1$ may also be used here, but 
we choose $\alpha =1$ for simplicity since $\kappa'(1)=\alpha$ (see discussion above).

\paragraph{Pooling operator.}
The last step to build the layer $x_k$ consists of pooling neighboring values to achieve local shift-invariance. We apply a linear convolution operator $A_k$ with a Gaussian filter of scale $\sigma_k$, $h_{\sigma_k}(u) := \sigma_k^{-d}h(u/\sigma_k)$, where $h(u) = (2 \pi)^{-d/2} \exp(-|u|^2/2)$. Then, for all $u$ in~$\Omega$,
\begin{equation}
   x_k(u) = A_k M_k P_k x_{\kmone}(u) = \int_{\Real^d} h_{\sigma_k}(u - v) M_k P_k x_{\kmone}(v) dv \in \Hc_k, \label{eq:xk}
\end{equation}
where the integral is a Bochner integral~\citep[see,][]{diestel,muandet2017kernel}.
By applying Schur's test to the integral operator~$A_k$ (see Appendix~\ref{sec:basic_tools}),
we obtain that the operator norm $\|A_k\|$ is less than $1$.
Thus, $x_k$ is in $L^2(\Omega, \Hc_k)$, with $\|x_k\| \leq \|M_k P_k x_{\kmone}\|$.
Note that a similar pooling operator is used in the scattering transform~\citep{mallat2012group}.

\paragraph{Multilayer construction and prediction layer.}
Finally, we obtain a multilayer representation by composing multiple times the
previous operators. In order to increase invariance with each layer and to increase the size of the receptive fields
(that is, the neighborhood of the original signal considered in a given patch), the size
of the patch~$S_k$ and pooling scale~$\sigma_k$ typically grow exponentially
with~$k$, with $\sigma_k$ and the patch size $\sup_{c\in S_k} |c|$ of the same order. With~$n$
layers, the maps~$x_n$ may then be written
\begin{equation}
   x_n \defin A_n M_n P_n A_{\nmone} M_{\nmone} P_{\nmone} \, \cdots \, A_1 M_1  P_1 x_0 ~\in~ L^2(\Omega,\Hc_n). \label{eq:final_repr}
\end{equation}
It remains to define a kernel from this representation, that will play the same role as the ``fully
connected'' layer of classical convolutional neural networks.
For that purpose, we simply consider the following linear kernel defined for all $x_0, x_0'$ in~$L^2(\Omega,\Hc_0)$
by using the corresponding
feature maps $x_n, x_n'$ in~$L^2(\Omega,\Hc_n)$ given by our multilayer construction~\eqref{eq:final_repr}:
\begin{equation}
\label{eq:linear_prediction_kernel}
   \Kcal_{n}(x_0,x_0') = \langle x_n, x_n' \rangle = \int_{u \in \Omega} \langle x_n(u), x_n'(u) \rangle du.
\end{equation}
Then, the RKHS $\Hc_{\mathcal K_n}$ of ${\mathcal K}_n$ contains all functions of the form $f(x_0) = \langle w, x_n \rangle$ with $w$ in $L^2(\Omega,\Hc_n)$ (see Appendix~\ref{sec:basic_tools}).

We note that one may also consider nonlinear kernels, such as a Gaussian kernel: 
\begin{equation}
\label{eq:gauss_prediction_kernel}
   \Kcal_{n}(x_0,x_0') = e^{-\frac{\alpha}{2}\|x_n-x_n'\|^2}.
\end{equation}
Such kernels are then associated to a RKHS denoted by~$\Hc_{n+1}$,
along with a kernel mapping $\varphi_{n+1}: L^2(\Omega,\Hc_n) \to \Hc_{n+1}$ which we call \emph{prediction layer}, so that
the final representation is given by~$\varphi_{n+1}(x_n)$ in $\Hc_{n+1}$.
We note that~$\varphi_{n+1}$ is non-expansive for the Gaussian kernel when~$\alpha \leq 1$ (see Section~\ref{sub:nonexp_proofs}),
and is simply an isometric linear mapping for the linear kernel.
Then, we have the relation
$
   {\mathcal K}_n(x_0, x_0') := \langle \varphi_{n+1}(x_n), \varphi_{n+1}(x_n') \rangle,
$
and in particular, the RKHS $\Hc_{\mathcal K_n}$ of ${\mathcal K}_n$ contains all functions of the form $f(x_0) = \langle w, \varphi_{n+1}(x_n) \rangle$ with $w$ in~$\Hc_{n+1}$, see Appendix~\ref{sec:basic_tools}.

\subsection{Signal Preservation and Discretization} 
\label{sub:kernel_approximation_and_discretization} 
In this section, we show that the multilayer kernel representation 
preserves all information about the signal at each layer, and besides, each feature
map~$x_k$ can be sampled on a discrete set with no loss of information. This
suggests a natural approach for discretization which will be discussed after the 
following lemma, whose proof is given in Appendix~\ref{sec:proofs}.
\begin{lemma}[Signal recovery from sampling]
\label{lemma:signal_recovery}
Assume that~$\Hcal_k$ contains all linear functions~$z \mapsto \langle g, z \rangle$
with~$g$ in~$\Pcal_k$ (this is true for all kernels~$K_k$ described in the previous section,
   according to Corollary~\ref{corollary:linear} in Section~\ref{sub:rkhs_activations} later);
   then, the signal~$x_{\kmone}$ can be recovered
from a sampling of~$x_k$ at discrete locations in a set~$\Omm$ as soon as $\Omm + S_k = \Omega$
   (\ie, the union of patches centered at these points covers~$\Omega$). It
   follows that~$x_k$ can be reconstructed from such a sampling.
\end{lemma}

The previous construction defines a kernel representation for general
signals in~$L^2(\Omega, \Hc_0)$, which is an abstract object defined for
theoretical purposes. In practice, signals are discrete, and it is
thus important to discuss the problem of discretization.
For clarity, we limit the presentation to 1-dimensional signals ($d = 1$),
but the arguments can easily be extended to higher
dimensions~$d$ when using box-shaped patches. Notation from the previous
section is preserved, but we add a bar on top of all discrete
analogues of their continuous counterparts. \eg, $\bar{x}_k$ is a
discrete feature map in $\ell^2(\Z,{\Hcbar_k})$ for some RKHS $\Hcbar_k$.

\paragraph{Input signals $x_0$ and~$\bar{x}_0$.}
Discrete signals acquired by a physical device may be seen as local
integrators of signals defined on a continuous domain (\eg, sensors from
digital cameras integrate the pointwise distribution of photons in a spatial and temporal window). Then, consider a signal~$x_0$ in
$L^2(\Omega,\Hcal_0)$ and $s_0$ a sampling interval.
By defining 
$\bar{x}_0$ in $\ell_2(\Z,\Hcal_0)$ such that $\bar{x}_0[n] = x_0(n s_0)$ for all $n$ in~$\Z$, it is thus
natural to assume that $x_0 \!=\! A_0 x$, where $A_0$ is a pooling operator (local
integrator) applied to an original continuous signal $x$. The role of $A_0$ is to prevent aliasing and
reduce high frequencies; typically, the scale $\sigma_0$ of~$A_0$ should be of the same magnitude
as $s_0$, which we choose to be $s_0=1$ without loss of generality.
This natural assumption is kept later for the stability analysis.

\paragraph{Multilayer construction.}
We now want to build discrete feature maps~$\bar{x}_k$ in~$\ell^2(\Z, \Hcbar_k)$ at each layer~$k$ involving subsampling with a factor~$s_k$ with respect to~$\bar{x}_{\kmone}$.
We now define the discrete analogues of the operators~$P_k$ (patch extraction), $M_k$ (kernel mapping), and~$A_k$ (pooling) as follows: for $n \in \Z$,
\begin{align*}
   \bar{P}_k \bar{x}_{\kmone}[n] &\defin \frac{1}{\sqrt{e_k}}(\bar{x}_{\kmone}[n], \bar{x}_{\kmone}[n + 1], \ldots, \bar{x}_{\kmone}[n + e_k - 1]) \in \bar{\Pcal}_k \defin \Hcbar_{\kmone}^{e_k} \\
   \bar{M}_k \bar{P}_k \bar{x}_{\kmone}[n] &\defin \bar{\varphi}_k( \bar{P}_k \bar{x}_{\kmone}[n] ) \in \Hcbar_k \\
   \bar{x}_k[n] \!=\! \bar{A}_k \bar{M}_k \bar{P}_k \bar{x}_{\kmone}[n] &\defin \frac{1}{\sqrt{s_k}}\sum_{m \in \Z} \bar{h}_k[n s_k - m] \bar{M}_k \bar{P}_k \bar{x}_{\kmone}[m] \!=\! (\bar{h}_k \ast \bar{M}_k \bar{P}_k \bar{x}_{\kmone})[n s_k] \in \Hcbar_k,
\end{align*}
where (i) $\bar{P}_k$ extracts a patch of size $e_k$ starting at position $n$
in $\bar{x}_{\kmone}[n]$,
which lives in the Hilbert space $\bar{\Pcal}_k$ defined as the
direct sum of $e_k$ times $\bar{\Hcal}_{\kmone}$; (ii)~$\bar{M}_k$ is a
kernel mapping identical to the continuous case, which preserves the norm, like $M_k$;
(iii)~$\bar{A}_k$ performs a convolution with a Gaussian filter and a subsampling operation with factor $s_k$. 
The next lemma shows that under mild assumptions, this construction preserves signal information.

\begin{lemma}[Signal recovery with subsampling]
\label{lemma:discretization}
   Assume that~$\Hcbar_k$ contains the linear functions~$z \mapsto \langle w, z
   \rangle$ for all~$w$ in $\bar{\Pcal}_{k}$ and that $e_k \geq s_k$.
   Then, $\bar{x}_{\kmone}$ can be recovered from~$\bar{x}_k$.
\end{lemma}
The proof is given in Appendix~\ref{sec:proofs}. The result relies on recovering patches using linear ``measurement'' functions
and deconvolution of the pooling operation.
While such a deconvolution operation can be unstable, it may be possible to obtain more stable recovery mechanisms
by also considering non-linear measurements, a question which we leave open.

\paragraph{Links between the parameters of the discrete and continuous models.}
Due to subsampling, the patch size in the continuous and discrete
models are related by a multiplicative factor. Specifically, a patch of size~$e_k$
with discretization corresponds to a patch~$S_k$ of diameter $e_k
s_{k-1}s_{k-2}\ldots s_1$ in the continuous case. The same holds true for the scale parameter~$\sigma_k$ of the Gaussian pooling. 

\subsection{Practical Implementation via Convolutional Kernel Networks}\label{subsec:ckn} 
Besides discretization, convolutional kernel networks add two modifications to implement in practice the image
representation we have described. First, it uses feature maps
with finite spatial support, which introduces border effects that we do not study~\citep[like][]{mallat2012group},
but which are negligible when dealing with large realistic images.
Second, CKNs use finite-dimensional approximations of the kernel feature map.
Typically, each RKHS's mapping is approximated by performing a projection onto a subspace of finite dimension,
which is a classical approach to make kernel methods work at large scale~\citep{fine2001efficient,smola2000sparse,williams2001}.
If we consider the kernel mapping~$\varphi_k: \Pcal_k \to \Hcal_k$ at layer~$k$, the orthogonal projection onto the finite-dimensional subspace~$\mathcal{F}_k = \vect(\varphi_k(z_1), \ldots, \varphi_k(z_{p_k})) \subseteq \Hc_k$, where the $z_i$'s are $p_k$ anchor points in $\Pcal_k$, is given by the linear operator~$\Pi_k: \Hcal_k \to \mathcal{F}_k$ defined for~$f$ in $\Hc_k$ by
\begin{equation}
\label{eq:ckn_proj}
   \Pi_k f \defin \sum_{1 \leq i,j \leq p_k} (K_{ZZ}^{-1})_{ij} \langle \varphi_k(z_i), f \rangle \varphi_k(z_j),
\end{equation}
where~$K_{ZZ}^{-1}$ is the inverse (or pseudo-inverse) of the~$p_k \times p_k$ kernel matrix~$[K_k(z_i, z_j)]_{ij}$.
As an orthogonal projection operator, $\Pi_k$ is non-expansive, \ie,~$\|\Pi_k\| \leq 1$.
We can then define the new approximate version~$\tilde{M}_k$ of the kernel mapping operator~$M_k$ by
\begin{equation}
\label{eq:ckn_kernel_operator}
   \tilde{M}_k P_k x_{\kmone}(u) \defin \Pi_k \varphi_k( P_k x_{\kmone}(u)) \in \Fcal_k.
\end{equation}
Note that all points in the feature map~$\tilde{M}_k P_k x_{\kmone}$ lie in the $p_k$-dimensional space~$\mathcal{F}_k \subseteq \Hcal_k$, which allows us to represent each point $\tilde{M}_k P_k x_{\kmone}(u)$ by the finite dimensional vector
\begin{equation}
\label{eq:ckn_proj_repr}
   \psi_k(P_k x_{\kmone}(u)) \defin K_{ZZ}^{-1/2} K_Z(P_k x_{\kmone}(u)) \in \R^{p_k},
\end{equation}
with~$K_Z(z) := (K_k(z_1, z), \ldots, K_k(z_{p_k}, z))^{\top}$; this
finite-dimensional representation preserves the Hilbertian inner product and
norm\footnote{We have $\langle \psi_k(z), \psi_k(z') \rangle_2 = \langle \Pi_k \varphi_k(z), \Pi_k \varphi_k(z') \rangle_{\Hc_k}$. See~\citet{mairal_end--end_2016} for details.} in $\Fcal_k$
so that $\|\psi_k(P_k x_{\kmone}(u))\|^2_2 = \|\tilde{M}_k P_k x_{\kmone}(u)\|^2_{\Hc_k}$.

Such a finite-dimensional mapping is compatible with the
multilayer construction, which builds~$\Hcal_k$ by manipulating points
from~$\Hcal_{\kmone}$. Here, the approximation 
provides points in $\Fcal_k \subseteq \Hcal_k$, which remain in~$\Fcal_k$ after pooling
since $\Fcal_k$ is a linear subspace. Eventually, the sequence of RKHSs
$\{\Hcal_k\}_{k \geq 0}$ is not affected by the finite-dimensional
approximation.
Besides, the stability results we will
present next are preserved thanks to the non-expansiveness of the projection.
In contrast, other kernel approximations such as random Fourier
features~\citep{rahimi2007} do not provide points in
the RKHS \citep[see][]{bach2017}, and their effect on the
functional space derived from the multilayer construction is unclear.

It is then possible to derive theoretical results for the CKN model, which
appears as a natural implementation of the kernel constructed previously; yet,
we will also show in Section~\ref{sec:link_with_cnns} that the results apply
more broadly to CNNs that are contained in the functional space associated to
the kernel. However, the stability of these CNNs depends on their RKHS norm,
which is hard to control. In contrast, for CKNs,
stability is typically controlled by the norm of the final prediction layer.

\section{Stability to Deformations and Group Invariance} 
\label{sec:stability}
In this section, we study the translation invariance and the stability under
the action of diffeomorphisms of the kernel representation described in
Section~\ref{sec:kernel_construction} for continuous signals.
In addition to translation invariance, it is desirable to have a representation
that is stable to small local deformations.
We describe such deformations using a $C^1$-diffeomorphism $\tau:\Omega \to \Omega$,
and let~$L_\tau$ denote the linear operator defined by $L_\tau x(u) = x(u - \tau(u))$.
We use a similar characterization of stability to the one introduced by~\citet{mallat2012group}:
the representation $\Phi(\cdot)$ is \emph{stable} under the action of diffeomorphisms
if there exist two non-negative constants $C_1$ and $C_2$ such that
\begin{equation}
\label{eq:stability_generic}
\|\Phi(L_\tau x) - \Phi(x) \| \leq (C_1 \|\nabla \tau\|_\infty + C_2 \|\tau\|_\infty)\|x\|,
\end{equation}
where $\nabla \tau$ is the Jacobian of $\tau$, $\|\nabla \tau\|_\infty \!:=\! \sup_{u \in \Omega} \|\nabla \tau(u)\|$, and $\|\tau\|_\infty \!:=\! \sup_{u \in \Omega} |\tau(u)|$.
The quantity~$\|\nabla \tau(u)\|$ measures the size of the deformation
at a location~$u$, and like~\citet{mallat2012group},
we assume the regularity condition $\|\nabla \tau\|_\infty \leq 1/2$,
which implies that the deformation is invertible~\citep{allassonniere2007towards,trouve2005local}
and helps us avoid degenerate situations.
In order to have a near-translation-invariant
representation, we want~$C_2$ to be small (a translation is a diffeomorphism
with $\nabla \tau=0$), and indeed we will show that~$C_2$ is proportional
to~$1/\sigma_n$, where~$\sigma_n$ is the scale of the last pooling layer, which
typically increases exponentially with the number of layers~$n$.
When~$\nabla \tau$ is non-zero, the diffeomorphism deviates from a translation,
producing local deformations controlled by~$\nabla \tau$.

\paragraph{Additional assumptions.}
In order to study the stability of the representation~\eqref{eq:final_repr}, we
assume that the input signal $x_0$ may be written as $x_0=A_0 x$, where $A_0$
is an initial pooling operator at scale $\sigma_0$, which allows us to control
the high frequencies of the signal in the first layer. As discussed previously
in Section \ref{sub:kernel_approximation_and_discretization}, this assumption is
natural and compatible with any physical acquisition device.
Note that~$\sigma_0$ can be taken arbitrarily small, so that this assumption does not limit the
generality of our results. Then, we are interested in understanding the stability of the representation
\begin{displaymath}
   \Phi_n(x) \defin A_n M_n P_n A_{\nmone}
   M_{\nmone} P_{\nmone} \, \cdots \, A_1 M_1  P_1 A_0 x.
\end{displaymath}
We do not consider a prediction layer~$\varphi_{n+1}$ here for simplicity, but note that if we add one on top of~$\Phi_n$, based on a linear of Gaussian kernel, then the stability
of the full representation $\varphi_{n+1} \circ \Phi_n$ immediately follows from that of~$\Phi_n$ thanks to the non-expansiveness of~$\varphi_{n+1}$ (see Section~\ref{sec:kernel_construction}).
Then, we make an assumption that relates the scale of the pooling operator at
layer $k-1$ with the diameter of the patch~$S_k$: we assume indeed that there exists
$\kappa > 0$ such that for all $k \geq 1$,
\begin{equation}
   \sup_{c \in S_k} |c| \leq \kappa \sigma_{k-1}. \tag{A2}\label{eq:patch_assumption}
\end{equation}
The scales~$\sigma_k$ are typically exponentially increasing with the layers~$k$, and characterize the
``resolution'' of each feature map. This assumption corresponds to considering patch sizes that
are adapted to these intermediate resolutions.
Moreover, the stability bounds we obtain hereafter increase with~$\kappa$,
which leads us to believe that small patch sizes
lead to more stable representations, something which matches well the trend of using small, 3x3
convolution filters at each scale in modern deep architectures \citep[\eg,~][]{simonyan2014very}.

Finally, before presenting our stability results, we recall a few properties of the operators involved in the representation $\Phi_n$, 
which are heavily used in the analysis.
\begin{center}
\fbox{
   \parbox{0.95\textwidth}{
\begin{enumerate}
   \item \textbf{Patch extraction operator}: $P_k$ is linear and preserves the norm;
   \item \textbf{Kernel mapping operator}: $M_k$ preserves the norm and is non-expansive;
   \item \textbf{Pooling operator}: $A_k$ is linear and non-expansive $\|A_k\| \leq 1$;
\end{enumerate}
}
}
\end{center}

The rest of this section is organized into three parts. 
We present the main stability results in Section~\ref{sub:stability_results},
explain their compatibility with kernel approximations in
Section~\ref{sub:stability_with_approximation},
and provide numerical experiment for demonstrating the stability of the kernel representation in
Section~\ref{sub:experiments}. 
Finally, we introduce mechanisms to achieve invariance to any group of transformations in Section~\ref{sec:global_invariance_to_group_actions}.

\subsection{Stability Results and Translation Invariance} 
\label{sub:stability_results}

Here, we show that our kernel representation~$\Phi_n$ satisfies the stability property~\eqref{eq:stability_generic},
with a constant~$C_2$ inversely proportional to~$\sigma_n$, thereby achieving near-invariance to translations.
The results are then extended to more general transformation groups in Section~\ref{sec:global_invariance_to_group_actions}.

\paragraph{General bound for stability.}
The following result gives an upper bound on the quantity of interest, $\| \Phi_n(L_\tau x) - \Phi_n(x) \|$, in terms of the norm of various linear operators which control how~$\tau$ affects each layer. An important object of study is the commutator of linear operators~$A$ and~$B$, which  is denoted by $[A, B] = AB - BA$.
\begin{proposition}[Bound with operator norms]
\label{prop:general_bound}
  For any $x$ in $L^2(\Omega,\Hcal_0)$, we have
\begin{equation}
\| \Phi_n(L_\tau x) - \Phi_n(x) \| \leq \left( \sum_{k=1}^{n} \|[P_k A_{k-1}, L_\tau]\|
		 + \|[A_n, L_\tau]\| + \|L_\tau A_n - A_n\| \right) \|x\|.
\end{equation}
\end{proposition}
For translations~$L_\tau x(u) = L_c x(u) = x(u - c)$, it is easy to see that patch extraction and pooling operators commute with~$L_c$ (this is also known as \emph{covariance} or \emph{equivariance} to translations), so that we are left with the term~$\|L_c A_n - A_n\|$, which should control translation invariance. For general diffeomorphisms~$\tau$, we no longer have exact covariance, but we show below that commutators are stable to~$\tau$, in the sense that $\|[P_k A_{k-1}, L_\tau]\|$ is controlled by~$\|\nabla \tau\|_\infty$, while $\|L_\tau A_n - A_n\|$ is controlled by~$\|\tau\|_\infty$ and decays with the pooling size~$\sigma_n$.

\paragraph{Bound on $\|[P_k A_{k-1}, L_\tau]\|$.}
We note that $P_k z$ can be identified with $(L_c z)_{c \in S_k}$ isometrically for all $z$ in $L^2(\Omega,\Hcal_{\kmone})$, since $\|P_k z\|^2 = \int_{S_k} \|L_c z\|^2 d \nu_k(c)$ by Fubini's theorem. Then,
\begin{align*}
\|P_k A_{k-1} L_\tau z - L_\tau P_k A_{k-1} z\|^2 &= \int_{S_k} \|L_c A_{k-1} L_\tau z - L_\tau L_c A_{k-1} z\|^2 d \nu_k(c) \\
	&\leq \sup_{c \in S_k} \|L_c A_{k-1} L_\tau z - L_\tau L_c A_{k-1} z\|^2,
\end{align*}
so that $\|[P_k A_{k-1}, L_\tau]\| \leq \sup_{c\in S_k} \|[L_c A_{k-1}, L_\tau]\|$. The following result lets us bound the commutator $\|[L_c A_{k-1}, L_\tau]\|$ when $|c| \leq \kappa \sigma_{k-1}$, which is satisfied under assumption~\eqref{eq:patch_assumption}.

\begin{lemma}[Stability of shifted pooling]
\label{lemma:stability}
Consider $A_\sigma$ the pooling operator with kernel $h_\sigma(u) = \sigma^{-d} h(u/\sigma)$. If $\|\nabla \tau\|_\infty \leq 1/2$, there exists a constant $C_1$ such that for any $\sigma$ and $|c| \leq \kappa \sigma$, we have
\begin{equation*}
\|[L_c A_\sigma, L_\tau]\| \leq C_1 \| \nabla \tau\|_\infty,
\end{equation*}
where $C_1$ depends only on $h$ and $\kappa$.
\end{lemma}
A similar result can be found in Lemma E.1 of \citet[][]{mallat2012group} for commutators of the form $[A_\sigma, L_\tau]$,
but we extend it to handle integral operators~$L_c A_\sigma$ with a shifted kernel.
The proof (given in Appendix~\ref{sub:stability_proof}) follows closely~\citet{mallat2012group} and relies on the fact that~$[L_c A_\sigma, L_\tau]$ is an integral operator
in order to bound its norm via Schur's test.
Note that~$\kappa$ can be made larger, at the cost of an increase of the constant~$C_1$ of the order $\kappa^{d+1}$.

\paragraph{Bound on $\|L_\tau A_n - A_n\|$.}
We bound the operator norm $\|L_\tau A_n - A_n\|$ in terms of~$\|\tau\|_\infty$ using the following result due to \citet[Lemma 2.11]{mallat2012group}, with~$\sigma = \sigma_n$:
\begin{lemma}[Translation invariance]
\label{lemma:transl_invariance}
If $\|\nabla \tau\|_\infty \leq 1/2$, we have
\begin{equation*}
\|L_\tau A_\sigma - A_\sigma\| \leq \frac{C_2}{\sigma} \|\tau\|_\infty,
\end{equation*}
with $C_2 = 2^d \cdot \|\nabla h\|_1$.
\end{lemma}
Combining Proposition~\ref{prop:general_bound} with Lemmas~\ref{lemma:stability} and~\ref{lemma:transl_invariance}, we now obtain the following~result:
\begin{theorem}[Stability bound]
\label{thm:stability}
   Assume~\eqref{eq:patch_assumption}. If $\|\nabla \tau\|_\infty \leq 1/2$, we have
\begin{equation}
\label{eq:stability_bound}
\| \Phi_n(L_\tau x) - \Phi_n(x) \| \leq \left( C_1 \left(1 + n \right) \|\nabla \tau\|_\infty
		 + \frac{C_2}{\sigma_n} \|\tau\|_\infty \right) \|x\|.
\end{equation}
\end{theorem}
This result matches the desired notion of stability in Eq.~(\ref{eq:stability_generic}), with a
translation-invariance factor that decays with~$\sigma_n$.
We discuss implications of our bound, and compare it with related work on stability in Section~\ref{sub:bound_discussion}.
We also note that our bound yields a worst-case guarantee on stability, in the sense that it holds for any signal~$x$.
In particular, making additional assumptions on the signal (\eg, smoothness) may lead to improved stability.
The predictions for a specific model may also be more stable than applying~\eqref{eq:cs} to our stability bound,
for instance if the filters are smooth enough.

\begin{remark}[Stability for Lipschitz non-linear mappings]
While the previous results require non-expansive non-linear mappings~$\varphi_k$,
it is easy to extend the result to the following more general condition
\begin{equation*}
\|\varphi_k(z) - \varphi_k(z')\| \leq \rho_k \|z - z'\| \quad \text{ and } \quad \|\varphi_k(z)\| \leq \rho_k \|z\|.
\end{equation*}
Indeed, the proof of Proposition~\ref{prop:general_bound} easily extends to this setting,
giving an additional factor~$\prod_k \rho_k$ in the bound~\eqref{eq:stability_generic}.
The stability bound~\eqref{eq:stability_bound} then becomes
\begin{equation}
\label{eq:stability_bound_lipschitz}
\| \Phi_n(L_\tau x) - \Phi_n(x) \| \leq \left(\prod_{k=1}^n \rho_k \right)\left( C_1 \left(1 + n \right) \|\nabla \tau\|_\infty
       + \frac{C_2}{\sigma_n} \|\tau\|_\infty \right) \|x\|.
\end{equation}
This will be useful for obtaining stability of CNNs with generic activations such as ReLU (see Section~\ref{sub:generic_activations}),
and this also captures the case of kernels with $\kappa_k'(1) > 1$ in Lemma~\ref{lemma:dp_kernels}.
\end{remark}

\subsection{Discussion of the Stability Bound (Theorem~\ref{thm:stability})}
\label{sub:bound_discussion}

In this section, we discuss the implications of our stability bound~\eqref{eq:stability_bound},
and compare it to related work on the stability of the scattering transform~\citep{mallat2012group}
as well as the work of~\citep{wiatowski2018mathematical} on more general convolutional models.

\paragraph{Role of depth.}
Our bound displays a linear dependence on the number of layers~$n$ in the stability constant $C_1 (1 + n)$.
We note that a dependence on a notion of
depth (the number of layers~$n$ here) also appears in~\cite{mallat2012group}, with a factor equal to
the maximal length of ``scattering paths'', and with the same condition $\|\nabla \tau\|_{\infty} \leq 1/2$.
Nevertheless, the number of layers is tightly linked to the patch sizes, and we now show how a deeper architecture
can be beneficial for stability.
Given a desired level of translation-invariance~$\sigma_f$
and a given initial resolution~$\sigma_0$, the above bound together with the discretization
results of Section~\ref{sub:kernel_approximation_and_discretization} suggest that one can obtain
a stable representation that preserves signal information by taking small patches at each layer
and subsampling with a factor equal to the patch size (assuming a patch size greater than one)
until the desired level of invariance is reached:
in this case we have~$\sigma_f / \sigma_0 \approx \kappa^n$,
where~$\kappa$ is of the order of the patch size, so that $n = O(\log(\sigma_f / \sigma_0) / \log(\kappa))$,
and hence the stability constant~$C_1 (1 + n)$ grows with~$\kappa$ as~$\kappa^{d+1} / \log(\kappa)$,
explaining the benefit of small patches, and thus of deeper models.

\paragraph{Norm preservation.}
While the scattering representation preserves the norm of the input signals when the length of scattering paths
goes to infinity, in our setting the norm may decrease with depth due to pooling layers.
However, we show in Appendix~\ref{sub:norm_preservation} that a part of the signal norm is still preserved,
particularly for signals with high energy in the low frequencies, as is the case for natural images~\citep[\eg,][]{torralba2003statistics}.
This justifies that the bounded quantity in~\eqref{eq:stability_bound} is relevant and non-trivial.
Nevertheless, we recall that despite a possible loss in norm, our (infinite-dimensional) representation~$\Phi(x)$ preserves
signal information, as discussed in Section~\ref{sub:kernel_approximation_and_discretization}.

\paragraph{Dependence on signal bandwidth.}
We note that our stability result crucially relies on the assumption~$\sigma_0 > 0$,
which effectively limits its applicability to signals with frequencies bounded by~$\lambda_0 \approx 1/\sigma_0$.
While this assumption is realistic in practice for digital signals, our bound degrades as~$\sigma_0$ approaches 0,
since the number of layers~$n$ grows as $\log (1 / \sigma_0)$, as explained above.
This is in contrast to the stability bound of~\citet{mallat2012group}, which holds uniformly over any such~$\sigma_0$,
thanks to the use of more powerful tools from harmonic analysis such as the Cotlar-Stein lemma,
which allows to control stability simultaneously at all frequencies thanks to the structure of the wavelet transform,
something which seems more challenging in our case due to the non-linearities separating different scales.

We note that it may be difficult to obtain meaningful stability results for an unbounded frequency support
given a fixed architecture,
without making assumptions about the filters of a specific model.
In particular, if we consider a model with a high frequency Fourier or cosine filter at the first layer,
supported on a large enough patch relative
to the corresponding wavelength, this will cause instabilities,
particularly if the input signal has isolated high frequencies~\citep[see, \eg,][]{bruna2013invariant}.
By the arguments of Section~\ref{sec:link_with_cnns}, such an unstable model~$g$ is in the RKHS, and we then have
that the final representation~$\Phi(\cdot)$ is also unstable, since
\begin{align*}
\|\Phi(L_\tau x) - \Phi(x)\| &= \sup_{f \in \Hc_{\Kcal_n}, \|f\| \leq 1} \langle f, \Phi(L_\tau x) - \Phi(x) \rangle \\
  &\geq \frac{1}{\|g\|} \langle g, \Phi(L_\tau x) - \Phi(x) \rangle = \frac{1}{\|g\|} (g(L_\tau x) - g(x)).
\end{align*}

\paragraph{Comparison with~\citet{wiatowski2018mathematical}.}
The work of~\citet{wiatowski2018mathematical} also studies deformation stability for generic convolutional
network models, however their ``deformation sensitivity'' result only shows that the representation is
as sensitive to deformations as the original signal, something which is also applicable here thanks to the
non-expansiveness of our representation. Moreover, their bound does not show the dependence on deformation size
(the Jacobian norm), and displays a translation invariance part that degrades linearly with~$1/\sigma_0$.
In contrast, the translation invariance part of our bound is independent of~$\sigma_0$,
and the overall bound only depends logarithmically on~$1/\sigma_0$,
by exploiting architectural choices such as pooling layers and patch sizes.

\subsection{Stability with Kernel Approximations} 
\label{sub:stability_with_approximation}

As in the analysis of the scattering transform of~\cite{mallat2012group}, we
have characterized the stability and shift-invariance of the data
representation for continuous signals, in order to give some intuition about
the properties of the corresponding discrete representation, which we have
described in Section~\ref{sub:kernel_approximation_and_discretization}.

Another approximation performed in the CKN model of
\cite{mairal_end--end_2016} consists of adding projection steps 
on finite-dimensional subspaces of the RKHS's layers, as discusssed in Section~\ref{subsec:ckn}.
Interestingly, the stability properties we have obtained previously are
compatible with these steps. We may indeed replace the
operator~$M_k$ with the operator $\tilde{M}_k z(u) = \Pi_k \varphi_k(z(u))$
for any map $z$ in $L^2(\Omega,\Pcal_k)$, instead of $M_k z(u) =
\varphi_k(z(u))$; $\Pi_k: \Hcal_k \to \Fcal_k$ is here an orthogonal projection operator
onto a linear subspace, given in~\eqref{eq:ckn_proj}. Then, $\tilde{M}_k$ does not
necessarily preserve the norm anymore, but $ \| \tilde{M}_k z\| \leq
\|z\|$, with a loss of information equal to~$\| M_k z - \tilde{M}_k z\|$ corresponding to the quality of
approximation of the kernel~$K_k$ on the points~$z(u)$. On the other hand, the
non-expansiveness of $M_k$ is satisfied thanks to the
non-expansiveness of the projection.
In summary, it is possible to show that the conclusions of Theorem~\ref{thm:stability} remain valid 
when adding the CKN projection steps at each layer, but some signal information is lost in the process.

\subsection{Empirical Study of Stability}
\label{sub:experiments}

In this section, we provide numerical experiments to demonstrate the stability properties
of the kernel representations defined in Section~\ref{sec:kernel_construction} on discrete images.

\begin{figure}[tb]
		\centering
		\includegraphics[width=0.6\textwidth]{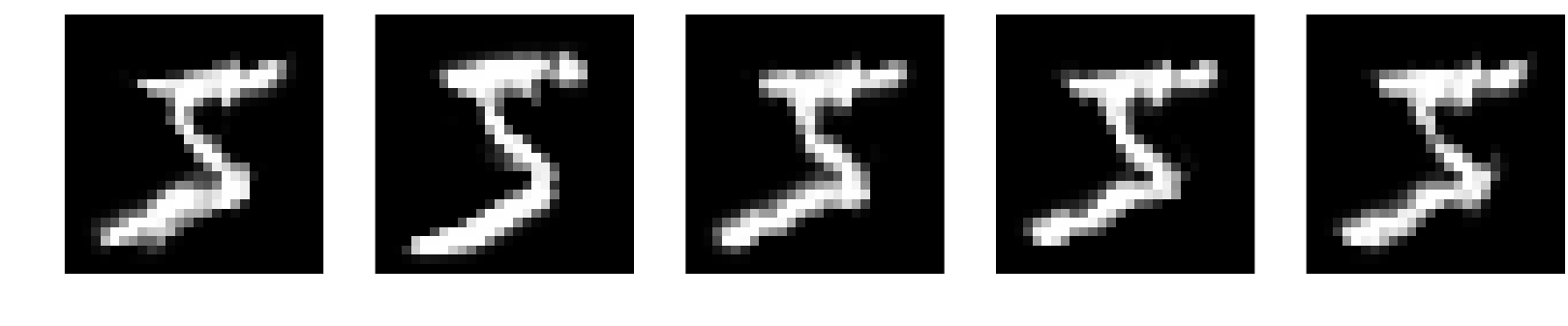}\\
		\includegraphics[width=0.6\textwidth]{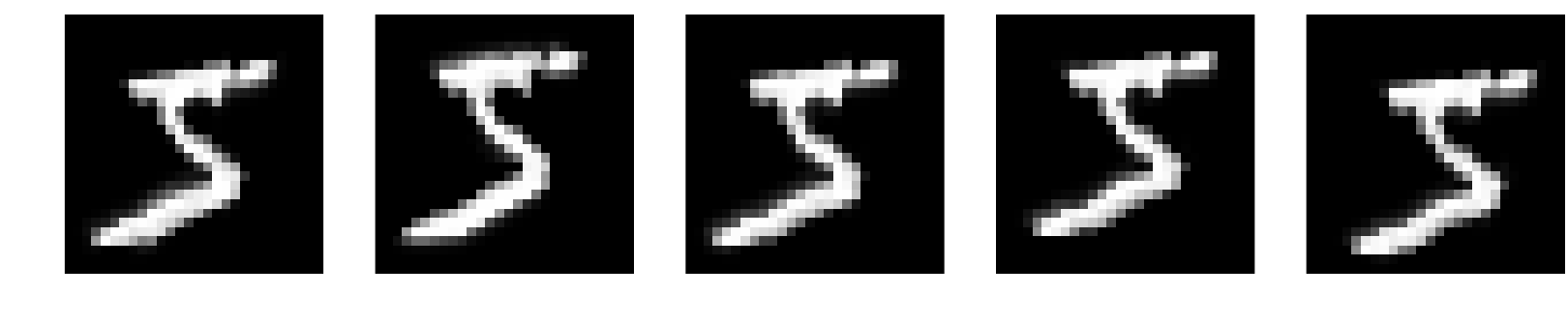}\\
		\includegraphics[width=0.6\textwidth]{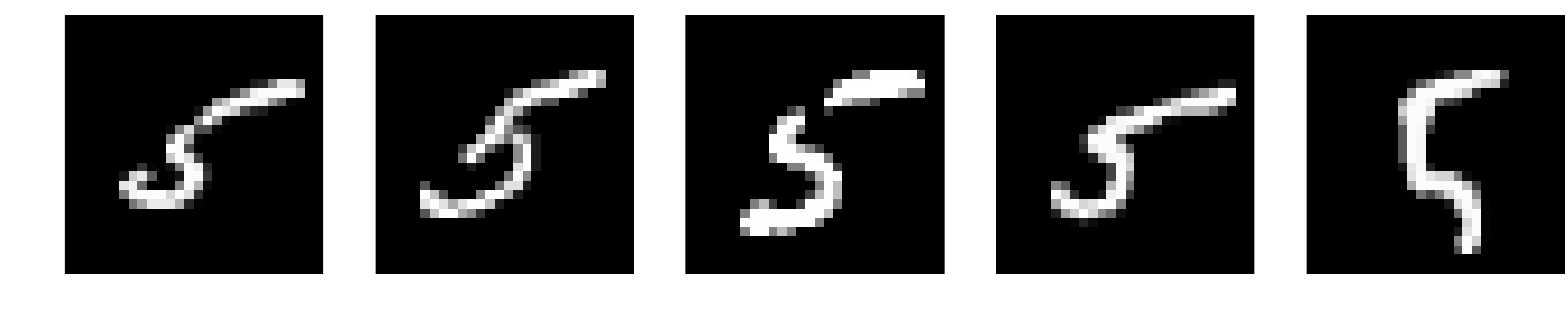}\\
		\includegraphics[width=0.6\textwidth]{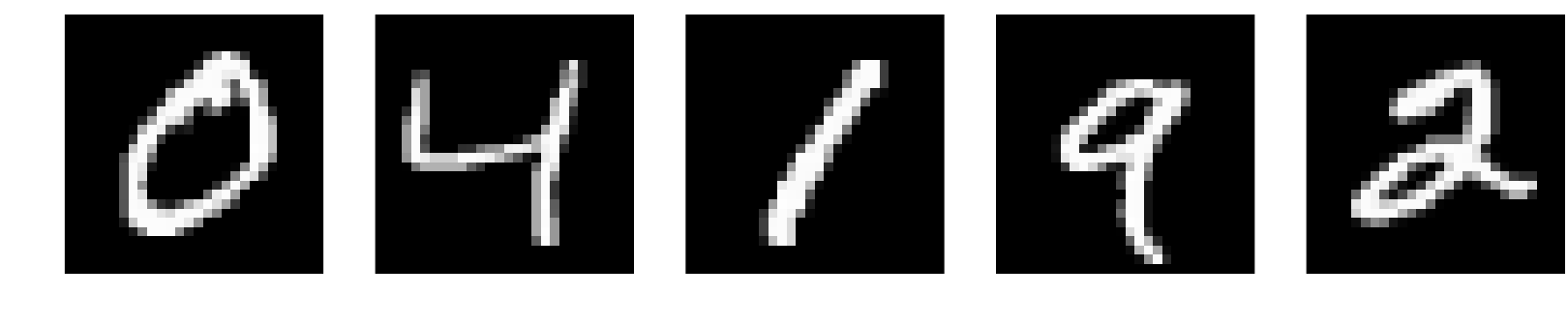}\\
	\caption{MNIST digits with transformations considered in our numerical study of stability.
	Each row gives examples of images from a set of digits that are compared to a reference image of a ``5''.
	From top to bottom:
	deformations with $\alpha = 3$;
	translations and deformations with $\alpha = 1$;
	digits from the training set with the same label ``5'' as the reference digit;
	digits from the training set with any label.
	}
	\label{fig:imnist}
\end{figure}

We consider images of handwritten digits from the Infinite MNIST dataset of~\citet{loosli2007training},
which consists of 28x28 grayscale MNIST digits augmented with small translations and deformations.
Translations are chosen at random from one of eight possible directions, while deformations
are generated by considering small smooth deformations $\tau$, and approximating $L_\tau x$
using a tangent vector field $\nabla x$ containing partial derivatives of the signal~$x$ along
the horizontal and vertical image directions.
We introduce a deformation parameter~$\alpha$ to control such deformations, which are then given by
\[
L_{\alpha\tau} x(u) = x(u - \alpha\tau(u)) \approx x(u) - \alpha \tau(u) \cdot \nabla x(u).
\]
Figure~\ref{fig:imnist} shows examples of different deformations, with various values of~$\alpha$,
with or without translations, generated from a reference image of the digit ``5''.
In addition, one may consider that a given reference image of a handwritten digit can
be deformed into different images of the same digit, and perhaps even into a different digit (\eg, a ``1'' may be deformed into a ``7'').
Intuitively, the latter transformation corresponds to a ``larger'' deformation than the former,
so that a prediction function that is stable to deformations should be preferable for a classification task.
The aim of our experiments is to quantify this stability, and to study how it is affected by architectural choices
such as patch sizes and pooling scales.

We consider a full kernel representation, discretized as described in Section~\ref{sub:kernel_approximation_and_discretization}.
We limit ourselves to 2 layers in order to make the computation of the full kernel tractable.
Patch extraction is performed with zero padding in order to preserve the size of the previous feature map.
We use a homogeneous dot-product kernel as in Eq.~\eqref{eq:dp_kernel_appx} with $\kappa(z) = e^{\rho (z - 1)}$, $\rho = 1 / (0.65)^2$.
Note that this choice yields $\kappa'(z) = \rho > 1$, giving an $\rho$-Lipschitz kernel mapping
instead of a non-expansive one as in Lemma~\ref{lemma:dp_kernels} which considers $\rho = 1$.
However, values of~$\rho$ larger than one typically lead to better empirical performance for classification~\citep{mairal_end--end_2016},
and the stability results of Section~\ref{sec:stability} are still valid with an additional factor $\rho^n$ (with $n=2$ here) in Eq.~\eqref{eq:stability_bound}.
For a subsampling factor~$s$, we apply a Gaussian filter with scale $\sigma = s / \sqrt{2}$ before downsampling.
Our C++ implementation for computing the full kernel given two images
is available at \url{https://github.com/albietz/ckn_kernel}.

\begin{figure}[tb]
	\centering
	\begin{subfigure}[c]{.32\textwidth}
		\includegraphics[width=\textwidth]{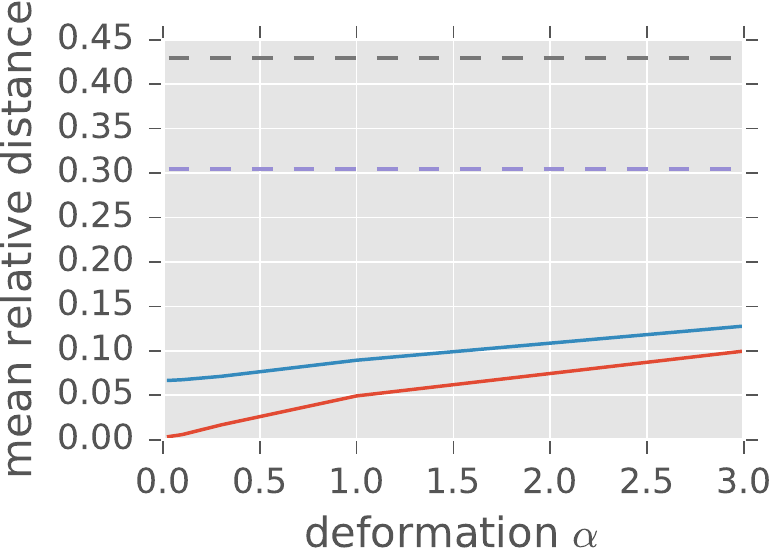}\vspace{-0.1cm}\caption{}\label{fig:imnist_deform}
	\end{subfigure}
	\begin{subfigure}[c]{.32\textwidth}
		\includegraphics[width=\textwidth]{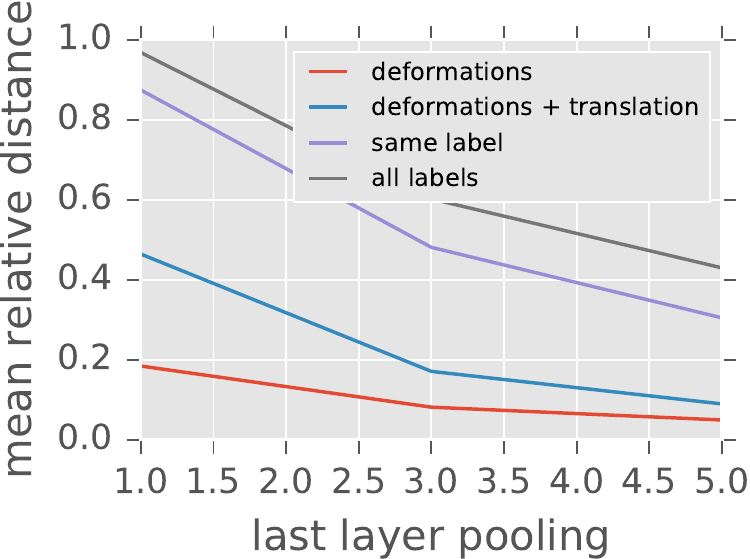}\vspace{-0.1cm}\caption{}\label{fig:imnist_inv}
	\end{subfigure}
	\begin{subfigure}[c]{.32\textwidth}
		\includegraphics[width=\textwidth]{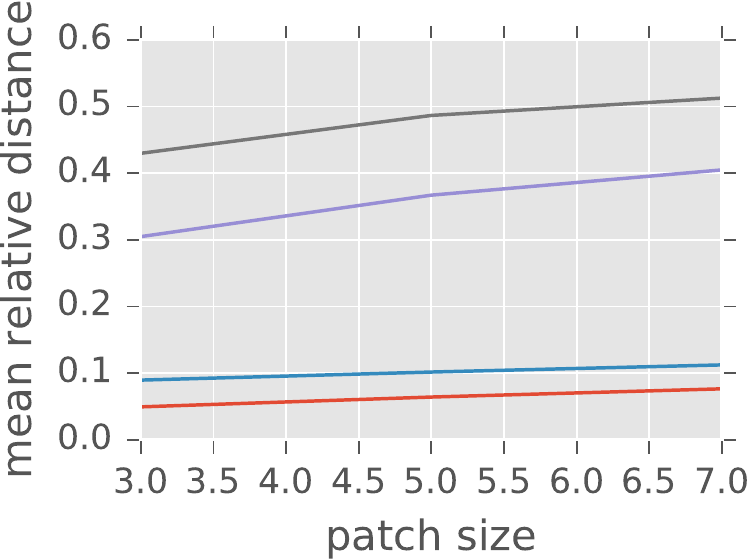}\vspace{-0.1cm}\caption{}\label{fig:imnist_patch}
	\end{subfigure}
	\caption{Average relative representation distance for various 2-layer models.
	Lines in the legend corresponds to rows of images in Figure~\ref{fig:imnist}.
	In (b-c), deformations are obtained with $\alpha = 1$.
	We show the impact on relative distance of:
	(a) the value of $\alpha$ in deformations, in~$\{0.01, 0.03, 0.1, 0.3, 1, 3\}$;
	(b) the subsampling factor of the final pooling layer, in~$\{1, 3, 5\}$;
	(c) the patch size, in~$\{3, 5, 7\}$.
	}
	\label{fig:imnist_rel_dist}
\end{figure}

In Figure~\ref{fig:imnist_rel_dist}, we show average relative distance in representation space between
a reference image and images from various sets of 20 images
(either generated transformations, or images appearing in the training set).
For a given architecture~$A$ and set~$S$ of images, the average relative distance to an image~$x$ is given by
\[
\frac{1}{|S|} \sum_{x' \in S} \frac{\|\Phi_A(x') - \Phi_A(x)\|}{\|\Phi_A(x)\|} = \frac{1}{|S|} \sum_{x' \in S} \frac{\sqrt{K_A(x, x) + K_A(x', x') - 2 K_A(x, x')}}{\sqrt{K_A(x, x)}},
\]
where~$\Phi_A$ denotes the kernel representation for architecture~$A$ and
$K_A(x, x')$ the corresponding kernel.
We normalize by~$\|\Phi_A(x)\|$ in order to reduce sensitivity to the choice of architecture.
We start with a $(3, 2)$-layer followed by a $(3, 5)$-layer,
where $(p, s)$ indicates a layer with patch size~$p$ and subsampling~$s$.
In Figure~\ref{fig:imnist_inv}, we vary the subsampling factor of the second layer,
and in Figure~\ref{fig:imnist_patch} we vary the patch size of both layers.

Each row of Figure~\ref{fig:imnist} shows digits and deformed versions.
Intuitively, it should be easier to deform an image of a handwritten 5 into a different image of a 5,
than into a different digit. Indeed, Figure~\ref{fig:imnist_rel_dist} shows
that the average relative distance for images
with different labels is always larger than for images with the same label,
which in turn is larger than for small deformations and translations of the reference image.

Adding translations on top of deformations increases distance in all cases, and Figure~\ref{fig:imnist_inv}
shows that this gap is smaller when using larger subsampling factors in the last layer.
This agrees with the stability bound~\eqref{eq:stability_bound}, which shows that a larger
pooling scale at the last layer increases translation invariance.
Figure~\ref{fig:imnist_deform} highlights the dependence of the distance on the deformation size~$\alpha$,
which is near-linear as in Eq.~\eqref{eq:stability_bound} (note that $\alpha$ controls the Jacobian of the deformation).
Finally, Figure~\ref{fig:imnist_patch} shows that larger patch sizes can make the representations less stable,
as discussed in Section~\ref{sec:stability}.


\subsection{Global Invariance to Group Actions} 
\label{sec:global_invariance_to_group_actions}

In Section~\ref{sub:stability_results}, we have seen how the kernel representation of Section~\ref{sec:kernel_construction}
creates invariance to translations by commuting with the action of translations at intermediate layers, and how the last pooling layer on the translation group governs the final level of invariance. It is often useful to encode invariances to different groups of transformations, such as rotations or reflections \citep[see, \eg,~][]{cohen2016group,mallat2012group,raj2017local,sifre2013rotation}. Here, we show how this can be achieved by defining adapted patch extraction and pooling operators that commute with the action of a transformation group~$G$ (this is known as group covariance or equivariance). We assume that~$G$ is locally compact such that we can define a left-invariant Haar measure~$\mu$---that is, a measure on~$G$ that satisfies $\mu(gS) = \mu(S)$ for any Borel set~$S \subseteq G$ and~$g$ in~$G$. We assume the initial signal~$x(u)$ is defined on~$G$, and we define subsequent feature maps on the same domain.
The action of an element~$g$ in~$G$ is denoted by~$L_g$, where $L_g x(u) = x(g^{-1} u)$.
In order to keep the presentation simple, we ignore some issues related to the general
construction in~$L^2(G)$ of our signals and operators,
which can be made more precise using tools from abstract harmonic analysis~\citep[\eg,][]{folland2016course}.

\paragraph{Extending a signal on~$G$.}
We note that the original signal is defined on a domain~$\Omega$
which may be different from the transformation group~$G$ that acts on~$\Omega$
(\eg, for 2D images the domain is~$\R^2$ but~$G$ may also include a rotation angle).
The action of~$g$ in~$G$ on the original signal defined on~$\Omega$, denoted~$\tilde x(\omega)$
yields a transformed signal $L_g \tilde x(\omega) = \tilde x(g^{-1} \cdot \omega)$, where~$\cdot$ denotes group action.
This requires an appropriate extension of the signal to~$G$ that preserves the meaning of signal transformations.
We make the following assumption:
every element~$\omega$ in $\Omega$ can be reached with a transformation $u_\omega$ in $G$
from a neutral element~$\epsilon$ in $\Omega$ (\eg, $\epsilon = 0$),
as $\omega = u_\omega \cdot \epsilon$.
Note that for 2D images ($d = 2$), this typically requires a group~$G$ that is ``larger'' than translations,
such as the roto-translation group, while it is not satisfied, for instance, for rotations only.
A similar assumption is made by~\citet{kondor2018generalization}.
Then, one can extend the original signal~$\tilde x$
by defining $x(u) := \tilde x(u \cdot \epsilon)$. Indeed, we then have
\[
L_g x(u_\omega) = x(g^{-1}u_\omega) = \tilde x((g^{-1} u_\omega) \cdot \epsilon) = \tilde x(g^{-1} \cdot \omega),
\]
so that the signal $(x(u_\omega))_{\omega \in \Omega}$ preserves the structure of~$\tilde x$.
We detail this below for the example of roto-translations on 2D images.
Then, we are interested in defining a layer---that is, a succession of patch
extraction, kernel mapping, and pooling operators---that commutes with $L_g$, in
order to achieve equivariance to~$G$.

\paragraph{Patch extraction.}
We define patch extraction as follows
\begin{equation*}
   P x(u) = (x(uv))_{v \in S}~~~~\text{for all}~~u \in G,
\end{equation*}
where $S \subset G$ is a patch shape centered at the identity element. $P$ commutes with~$L_g$ since
\begin{equation*}
P L_g x(u) = (L_g x(uv))_{v\in S} = (x(g^{-1}uv))_{v\in S} = P x(g^{-1}u) = L_g P x(u).
\end{equation*}

\paragraph{Kernel mapping.}
The pointwise operator~$M$ is defined exactly as in Section~\ref{sec:kernel_construction}, and thus commutes with~$L_g$.

\paragraph{Pooling.}
The pooling operator on the group~$G$ is defined  by
\begin{equation*}
A x(u) = \int_G x(uv) h(v) d \mu(v) = \int_G x(v) h(u^{-1}v) d \mu(v), 
\end{equation*}
where $h$ is a pooling filter typically localized around the identity element. The construction is similar to~\citet{raj2017local} and it is easy to see from the first expression of $A x(u)$ that $A L_g x(u) = L_g A x(u)$, making the pooling operator $G$-equivariant.
One may also pool on a subset of the group by only integrating over the subset in the first expression,
an operation which is also $G$-equivariant.

In our analysis of stability in Section~\ref{sub:stability_results},
we saw that inner pooling layers are useful to guarantee stability to local deformations,
while global invariance is achieved mainly through the last pooling layer.
In some cases, one only needs stability to a subgroup of~$G$, while achieving invariance to the whole group,
\eg, in the roto-translation group~\citep{oyallon2015deep,sifre2013rotation},
one might want invariance to a global rotation but stability to local translations.
Then, one can perform patch extraction and pooling just on the subgroup to stabilize (\eg, translations)
in intermediate layers, while pooling on the entire group at the last layer to achieve the global group invariance.

\paragraph{Example with the roto-translation group.}
We consider a simple example on 2D images where one wants global invariance to rotations
in addition to near-invariance and stability to translations as in Section~\ref{sub:stability_results}.
For this, we consider the roto-translation group~\citep[see, \eg,][]{sifre2013rotation},
defined as the \emph{semi-direct} product of translations~$\R^2$ and rotations~$SO(2)$,
denoted by $G = \R^2 \rtimes SO(2)$, with the following group operation
\[
g g' = (v + R_\theta v', \theta + \theta'),
\]
for $g = (v, \theta)$, $g' = (v', \theta')$ in~$G$, where~$R_\theta$ is a rotation matrix in~$SO(2)$.
The element $g = (v, \theta)$ in $G$ acts on a location~$u \in \R^2$ by combining a rotation and a translation:
\begin{align*}
g \cdot u &= v + R_\theta u \\
g^{-1} \cdot u &=(-R_{-\theta}v, -\theta) \cdot u = R_{-\theta}(u - v).
\end{align*}
For a given image $\tilde x$ in $L^2(\R^2)$,
our equivariant construction outlined above
requires an extension of the signal to the group~$G$.
We consider the Haar measure given by $d\mu((v, \theta)) := dv d \mu_c(\theta)$,
where~$dv$ is the Lebesgue measure on~$\R^2$ and~$d\mu_c$ the normalized Haar measure on the unit circle.
Note that~$\mu$ is left-invariant, since the determinant of rotation matrices
that appears in the change of variables is 1.
We can then define $x$ by $x((u, \eta)) := \tilde x (u)$ for any angle $\eta$, which is in~$L^2(G)$ and
preserves the definition of group action on the original signal $\tilde x$ since
\[
L_g x ((u, \eta)) = x(g^{-1}(u, \eta)) = x((g^{-1}\cdot u, \eta - \theta)) = \tilde x(g^{-1} \cdot u) = L_g \tilde x(u).
\]
That is, we can study the action of~$G$ on 2D images in~$L^2(\R^2)$ by studying
the action on the extended signals in~$L^2(G)$ defined above.

We can now define patch extraction and pooling operators $P, A : L^2(G) \to L^2(G)$ only on the translation subgroup,
by considering a patch shape $S = \{(v, 0)\}_{v \in \tilde S} \subset G$ with $\tilde S \subset \R^2$ for~$P$,
and defining pooling by
$
A x(g) = \int_\Omega x(g(v,0)) h(v) dv,
$
where~$h$ is a Gaussian pooling filter with scale~$\sigma$ defined on~$\R^2$.

The following result, proved in Appendix~\ref{sec:proofs}, shows analogous results
to the stability lemmas of Section~\ref{sub:stability_results} for the operators $P$ and $A$.
For a diffeomorphism~$\tau$, we denote by $L_\tau$ the action operator given by
$L_\tau x((u, \eta)) = x((\tau(u), 0)^{-1}(u, \eta)) = x((u - \tau(u), \eta))$.
\begin{lemma}[Stability with roto-translation patches]
\label{lemma:roto_stability}
If $\|\nabla \tau\|_\infty \leq 1/2$, and the following condition holds $\sup_{c \in \tilde S} |c| \leq \kappa \sigma$, we have
\begin{equation*}
\|[P A, L_\tau]\| \leq C_1 \| \nabla \tau\|_\infty,
\end{equation*}
with the same constant~$C_1$ as in Lemma~\ref{lemma:stability}, which depends on~$h$ and~$\kappa$.
Similarly, we have
\begin{equation*}
\|L_\tau A - A\| \leq \frac{C_2}{\sigma} \|\tau\|_\infty,
\end{equation*}
with~$C_2$ as defined in Lemma~\ref{lemma:transl_invariance}.
\end{lemma}

By constructing a multi-layer representation~$\Phi_n(x)$ in $L^2(G)$ using similar operators at each layer,
we can obtain a similar stability result to Theorem~\ref{thm:stability}.
By adding a global pooling operator~$A_c : L^2(G) \to L^2(\R^2)$ after the last layer,
defined, for~$x \in L^2(G)$, as
\[
A_c x(u) = \int x((u, \eta)) d \mu_c(\eta),
\]
we additionally obtain global invariance to rotations,
as shown in the following theorem.
\begin{theorem}[Stability and global rotation invariance]
\label{thm:roto_stability}
   Assume~\eqref{eq:patch_assumption} for patches~$\tilde S$ at each layer.
   Define the operator $L_{(\tau, \theta)} x((u, \eta)) = x((\tau(u), \theta)^{-1} (u, \eta))$,
   and define the diffeomorphism $\tau_\theta: u \mapsto R_{-\theta} \tau(u)$.
   If $\|\nabla \tau\|_\infty \leq 1/2$, we have
\begin{align*}
\| A_c \Phi_n(L_{(\tau, \theta)} x) - A_c \Phi_n(x) \|
	&\leq \|\Phi_n(L_{R_{\tau_\theta}} x) - \Phi_n(x) \|\\
	&\leq \left( C_1 \left(1 + n \right) \|\nabla \tau\|_\infty
		 + \frac{C_2}{\sigma_n} \|\tau\|_\infty \right) \|x\|.
\end{align*}
\end{theorem}
We note that a similar result may be obtained when $G = \Omega \rtimes H$, where~$H$ is any compact group,
with a possible additional dependence on how elements of~$H$ affect the size of patches.


\section{Link with Existing Convolutional Architectures} 
\label{sec:link_with_cnns}
In this section, we study the functional spaces (RKHS) that arise from our multilayer kernel representation,
and examine the connections with more standard convolutional architectures.
The motivation of this study is that if a CNN model~$f$ is in the RKHS, then it can be written in a ``linearized''
form $f(x) = \langle f, \Phi(x) \rangle$, so that our study of stability of the kernel representation~$\Phi$
extends to predictions using $|f(x) - f(x')| \leq \|f\| \|\Phi(x) - \Phi(x')\|$.

We begin by considering in Section~\ref{subsec:kernels_kk} the intermediate kernels~$K_k$,
showing that their RKHSs contain simple neural-network-like functions defined on patches with smooth activations,
while in Section~\ref{sub:cnns_rkhs} we show that a certain class of generic CNNs
are contained in the RKHS $\Hc_{\Kcal_n}$ of the full multilayer kernel~$\Kcal_n$ and characterize their norm.
This is achieved by considering particular functions in each intermediate RKHS defined in terms of the convolutional
filters of the CNN. A consequence of these results is that
our stability and invariance properties from Section~\ref{sec:stability} are valid for this broad class of CNNs.

\subsection{Activation Functions and Kernels~$K_k$}\label{subsec:kernels_kk} 

\label{sub:rkhs_activations}
Before introducing formal links between our kernel representation and classical
convolutional architectures, we study in more details the kernels~$K_k$
described in Section~\ref{sec:kernel_construction} and their RKHSs~$\Hcal_k$.
In particular, we are interested in characterizing which types of functions live in~$\Hcal_k$. 
The next lemma extends some results of~\citet{zhang2016l1,zhang2016convexified},
originally developed for the inverse polynomial and Gaussian kernels; it shows
that the RKHS may contain simple ``neural network'' functions with activations~$\sigma$
that are smooth enough. 
\begin{lemma}[Activation functions and RKHSs~$\Hcal_k$]
\label{lemma:activations}
   Let $\sigma: [-1,1] \to \Real$ be a function that admits a
   polynomial expansion $\sigma(u) := \sum_{j=0}^\infty a_j u^j$.
   Consider a kernel~$K_k$ from Section~\ref{sec:kernel_construction}, given in~\eqref{eq:dp_kernel_appx},
   with $\kappa_k(u)=\sum_{j=0}^{\infty} b_j u^j$, and $b_j \geq 0$ for all~$j$.
   Assume further that $a_j = 0$ whenever $b_j = 0$, and define the function $C_\sigma^2(\lambda^2) := \sum_{j=0}^\infty
   (a_j^2/b_j) \lambda^{2j}$.
   Let $g$ in~$\Pcal_k$ be such that $C_\sigma^2(\|g\|^2) < \infty$.
 Then, the RKHS $\Hcal_k$ contains the function
\begin{equation}
\label{eq:activations}
  f: z \mapsto \|z\| \sigma(\langle g, z \rangle / \|z\|),
\end{equation}
and its norm satisfies~$\|f\| \leq C_\sigma(\|g\|^2)$.
\end{lemma}
Noting that for all examples of~$\kappa_k$ given in Section~\ref{sec:kernel_construction}, we have $b_1 > 0$,
this result implies the next corollary, which was also found to be useful in our analysis.
\begin{corollary}[Linear functions and RKHSs]\label{corollary:linear}
   The RKHSs~$\Hc_k$ for the examples of~$\kappa_k$ given in Section~\ref{sec:kernel_construction}
   contain all linear functions of the form $z \mapsto \langle g, z \rangle$ with $g$ in~$\Pcal_k$.
\end{corollary}

The previous lemma shows that for many choices of smooth functions $\sigma$,
the RKHS $\Hcal_k$ contains the functions of the form~\eqref{eq:activations}.
While the non-homogeneous functions $z \mapsto \sigma(\langle g, z \rangle)$
are standard in neural networks, the homogeneous variant is not.
Yet, we note that (i)~the most successful
activation function, namely rectified linear units, is homogeneous---that is, 
$\text{relu}(\langle g, z \rangle) = \|z\|\text{relu}( \langle g, z \rangle /\|z\|)$; 
(ii)~while $\text{relu}$ is nonsmooth and thus not in our RKHSs, there exists a
smoothed variant that satisfies the conditions of Lemma~\ref{lemma:activations}
for useful kernels. As noticed by~\citet{zhang2016l1,zhang2016convexified}, this
is for instance the case for the inverse polynomial kernel. In
Figure~\ref{fig:relu}, we plot and compare these different variants of ReLU.

\begin{figure}[tb]
	\centering
	\includegraphics[width=.48\textwidth]{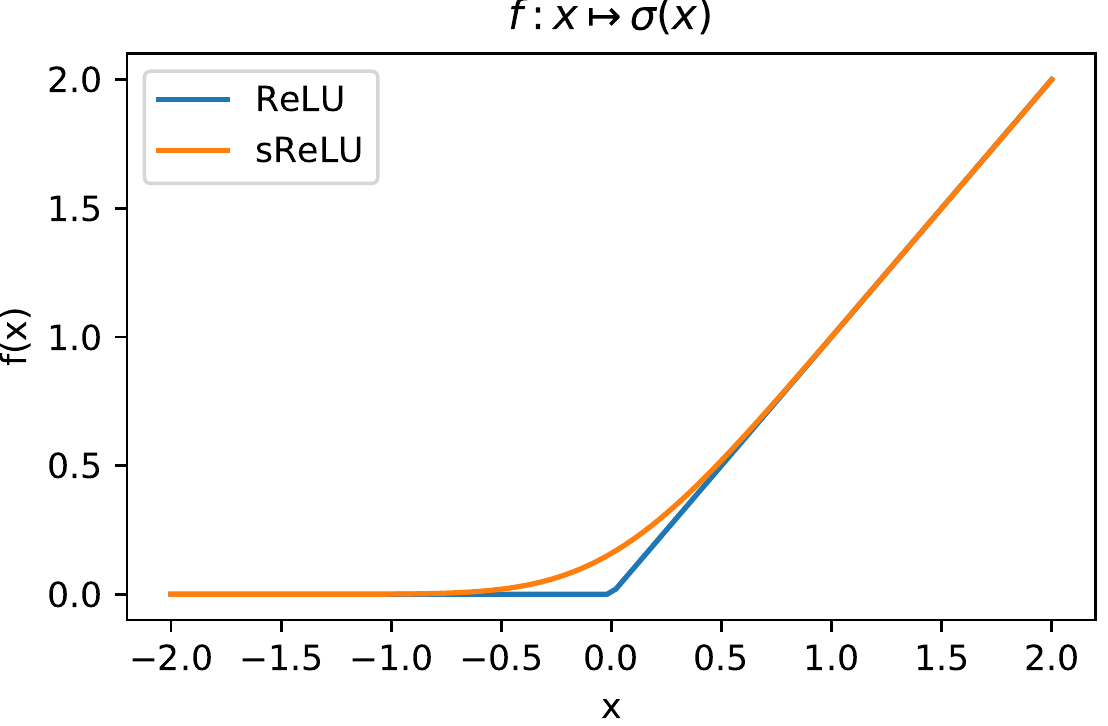}
	\includegraphics[width=.48\textwidth]{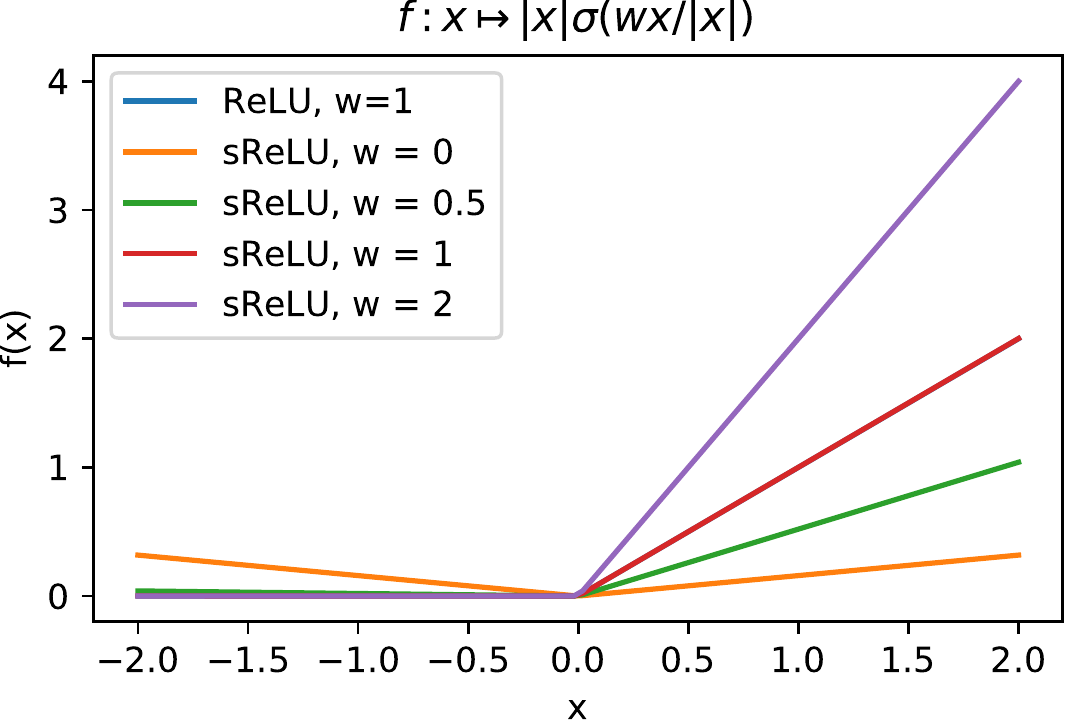}
        \vs
        \vs
        \vs
	\caption{Comparison of one-dimensional functions obtained with relu and smoothed relu (sReLU) activations. (Left) non-homogeneous setting of~\citet{zhang2016l1,zhang2016convexified}. (Right) our homogeneous setting, for different values of the parameter~$w$. Note that for~$w \geq 0.5$, sReLU and ReLU are indistinguishable.}
	\label{fig:relu}
        \vs
        \vs
        \vs
\end{figure}

\subsection{Convolutional Neural Networks and their Complexity}
\label{sub:cnns_rkhs}

We now study the connection between the kernel representation defined in Section~\ref{sec:kernel_construction}
and CNNs. Specifically, we show that the RKHS of the final kernel~$\Kcal_n$ obtained from
our kernel construction contains a set of
CNNs on continuous domains with certain types of smooth homogeneous activations.
An important consequence is that the stability results of
previous sections apply to this class of CNNs, although the stability depends on the
RKHS norm, as discussed later in Section~\ref{sec:discussion_conclusion}.
This norm also serves as a measure of model complexity, thus controlling both generalization and stability.

\vs
\paragraph{CNN maps construction.}
We now define a CNN function~$f_\activ$ that takes as input an image~$z_0$ in $L^2(\Omega, \R^{p_0})$
with~$p_0$ channels, and build a sequence of feature maps, represented at layer~$k$ as a
function~$z_k$ in $L^2(\Omega,\R^{p_k})$ with $p_k$ 
channels; the map~$z_k$ is obtained from~$z_{\kmone}$ by performing linear convolutions with a set of filters
$(w_k^i)_{i=1,\ldots,p_k}$, followed by a pointwise activation function~$\sigma$ to
obtain an intermediate feature map~$\tilde{z}_k$, then by applying 
a linear pooling filter.
Note that each $w_k^i$ is in $L^2(S_k,\R^{p_{\kmone}})$, with
channels denoted by~$w_k^{ij}$ in $L^2(S_k,\R)$.
Formally, the intermediate map $\tilde{z}_k$ in $L^2(\Omega,\R^{p_k})$ is obtained by
\begin{equation}
   \tilde{z}_k^i(u) = n_k(u) \sigma\! \left( \langle w_k^i, P_k z_{\kmone}(u) \rangle / n_k(u) \right),  \label{eq:cnn}
\end{equation}
where $\tilde{z}_k(u) = (\tilde{z}_k^1(u),\ldots,\tilde{z}_k^{p_k}(u))$ is in $\Real^{p_k}$, and $P_k$ is a patch
extraction operator for finite-dimensional maps.  The activation involves a pointwise non-linearity~$\sigma$
along with a quantity $n_k(u) \defin \|P_k x_{k-1}(u)\|$
in~(\ref{eq:cnn}), which is due to the homogenization, and which is
independent of the filters $w_k^i$.
Finally, the map $z_k$ is obtained by using a pooling operator as in Section~\ref{sec:kernel_construction}, with $z_k = A_k \tilde{z}_k$, and $z_0=x_0$.

\vs
\paragraph{Prediction layer.}
For simplicity, we consider the case of a linear fully connected prediction layer.
In this case, the final CNN prediction function~$f_{\sigma}$ is given by
\begin{equation*}
   f_\activ (x_0) = \langle w_{n+1},  z_n \rangle,
\end{equation*}
with parameters~$w_{n+1}$ in $L^2(\Omega, \R^{p_n})$.
We now show that such a CNN function is contained in the RKHS of the kernel~$\Kcal_n$ defined
in~\eqref{eq:linear_prediction_kernel}.

\paragraph{Construction in the RKHS.}
The function~$f_\sigma$ can be constructed recursively from intermediate functions that lie in the
RKHSs~$\Hc_k$, of the form~\eqref{eq:activations}, for appropriate activations~$\activ$.
Specifically, we define initial quantities
 $f_1^{i}$ in $\mathcal{H}_1$ and $g^i_1$ in $\Pcal_1$ for~$i = 1, \ldots, p_1$ such that
\begin{align*}
   g_1^i &= w_1^{i}  \in L^2(S_1,\Real^{p_0}) = L^2(S_1,\Hcal_0) = \Pcal_1,\\
f_1^i(z) &= \|z\| \activ(\langle g^0_i, z \rangle / \|z\|) \quad \text{ for } z \in \Pcal_1, 
\end{align*}
and we define, from layer $\kmone$, the quantities $f_k^{i}$ in $\mathcal{H}_k$ and $g_{k}^i$ in $\Pcal_k$
for $i = 1, \ldots, p_k$:
\begin{align*}
   g_{k}^i(v) &= \sum_{j=1}^{p_{\kmone}} w_{k}^{ij}(v) f_{\kmone}^j ~~~\text{where}~~~ w_k^i(v) = (w_k^{ij}(v))_{j=1,\ldots,p_{\kmone}},\\
f_k^i(z) &= \|z\| \activ(\langle g_{k}^i, z \rangle / \|z\|) \quad \text{ for } z \in \Pcal_k.
\end{align*}
For the linear prediction layer, we define~$g_\sigma$ in $L^2(\Omega, \Hc_n)$ by:
\begin{align*}
   g_{\sigma}(u) &= \sum_{j=1}^{p_n} w_{n+1}^j(u) f_n^j \quad \text{ for all } u\in \Omega,
\end{align*}
so that the function~$f: x_0 \mapsto \langle g_\sigma, x_n \rangle$ is in the RKHS of~$\Kcal_n$,
where~$x_n$ is the final representation given in Eq.~\eqref{eq:final_repr}.
In Appendix~\ref{sub:cnn_construction_rkhs}, we show that $f = f_\sigma$, which implies that
the CNN function~$f_\sigma$ is in the RKHS.
We note that a similar construction for fully connected multilayer networks with
constraints on weights and inputs was given by~\citet{zhang2016l1}.

\paragraph{Norm of the CNN~$f_\sigma$.}
We now study the RKHS norm of the CNN constructed above.
This quantity is important as it controls the stability and invariance of the predictions of a learned model through~\eqref{eq:cs}.
Additionally, the RKHS norm provides a way to control model complexity, and can lead to generalization bounds, \eg, through Rademacher complexity and margin bounds~\citep{boucheron2005theory,shalev2014understanding}.
In particular, such results rely on the following upper bound on the empirical Rademacher complexity of a function class with bounded RKHS norm $\mathcal{F}_\lambda = \{f \in \Hc_{\mathcal K_n} : \|f\| \leq \lambda \}$, for a dataset $\{x^{(1)}, \ldots, x^{(N)}\}$:
\begin{equation}
R_N(\mathcal F_\lambda) \leq \frac{\lambda \sqrt{\frac{1}{N} \sum_{i=1}^N \Kcal_n(x^{(i)}, x^{(i)})}}{\sqrt{N}}.
\end{equation}
The bound remains valid when only considering CNN functions in~$\mathcal F_\lambda$ of the form $f_\sigma$, since such a function class is contained in~$\mathcal F_\lambda$.
If we consider a binary classification task with training labels~$y^{(i)}$ in~$\{-1, 1\}$,
on can then obtain a margin-based bound for any function~$f_N$ in~$\mathcal F_\lambda$ obtained from the training set
and any margin~$\gamma > 0$: with probability $1- \delta$, we have~\citep[see, \eg,][]{boucheron2005theory}
\begin{equation}
\label{eq:margin_bound}
L(f_N) \leq L_N^\gamma(f_N) + O \left( \frac{\lambda \sqrt{\frac{1}{N} \sum_{i=1}^N \Kcal_n(x^{(i)}, x^{(i)})}}{\gamma \sqrt{N}} + \sqrt{\frac{\log(1/\delta)}{N}} \right),
\end{equation}
with
\begin{align*}
L(f) &= \mathbb{P}_{(x, y) \sim \mathcal{D}}(y f(x) < 0) \\
L_N^\gamma(f) &= \frac{1}{N} \sum_{i=1}^N \1\{y^{(i)} f(x^{(i)}) < \gamma \},
\end{align*}
where $\mathcal{D}$ is the distribution of data-label pairs $(x^{(i)}, y^{(i)})$.
Intuitively, the margin~$\gamma$ corresponds to a level of confidence, and~$L_N^\gamma$ measures
training error when requiring confident predictions.
Then, the bound on the gap between this training error and the true expected error~$L(f_N)$ becomes larger
for small confidence levels, and is controlled by the model complexity~$\lambda$ and the sample size~$N$.

Note that the bound requires a fixed value of~$\lambda$ used during training, but in practice,
learning under a constraint~$\|f\| \leq \lambda$ can be difficult, especially for CNNs
which are typically trained with stochastic gradient descent with little regularization.
However, by considering values of~$\lambda$ on a logarithmic scale and taking a union bound,
one can obtain a similar bound with $\|f_N\|$ instead of~$\lambda$, up to logarithmic factors~\citep[see, \eg,][Theorem 26.14]{shalev2014understanding},
where~$f_N$ is obtained from the training data.
We note that various authors have recently considered other norm-based complexity measures to control
the generalization of neural networks with more standard activations \citep[see, \eg,~][]{bartlett2017spectrally,liang2017fisher,neyshabur2017exploring,neyshabur2015norm}.
However, their results are typically obtained for fully connected networks on finite-dimensional inputs,
while we consider CNNs for input signals defined on continuous domains.
The next proposition (proved in Appendix~\ref{sub:cnn_construction_rkhs}) characterizes the norm of~$f_\sigma$
in terms of the $L^2$ norms of the filters~$w_k^{ij}$,
and follows from the recursive definition
of the intermediate RKHS elements~$f_k^i$.

\begin{proposition}[RKHS norm of CNNs]\label{prop:cnns}
Assume the activation~$\activ$ satisfies~$C_\activ(a) < \infty$ for all~$a \geq 0$, where~$C_\activ$ is
defined for a given kernel in Lemma~\ref{lemma:activations}.
Then, the CNN function~$f_\activ$ defined above is in the RKHS~$\Hcal_{\mathcal{K}_n}$, with norm
\begin{align*}
   \|f_\activ\|^2 \leq p_n \sum_{i=1}^{p_n} \|w_{n+1}^i\|_2^2 B_{n,i},
\end{align*}
   where $B_{n,i}$ is defined by~$B_{1,i} = C_\activ^2(\|w_1^i\|_2^2 )$ and $B_{k,i} = C_\activ^2 \left(p_{\kmone} \sum_{j=1}^{p_{\kmone}} \|w_k^{ij}\|^2_2 B_{\kmone,j} \right)$.
\end{proposition}
Note that this upper bound need not grow exponentially with depth when the filters have small norm and $C_\sigma$ takes small values around zero.
However, the dependency of the bound on the number of feature maps~$p_k$ of each layer~$k$ may not be satisfactory
in situations where the number of parameters is very large, which is common in successful deep learning architectures.
The following proposition removes this dependence, relying instead on matrix spectral norms.
Similar quantities have been used recently to obtain useful generalization bounds for neural networks~\citep{bartlett2017spectrally,neyshabur2017pac}.

\begin{proposition}[RKHS norm of CNNs using spectral norms]\label{prop:cnns_spectral}
Assume the activation~$\activ$ satisfies~$C_\activ(a) < \infty$ for all~$a \geq 0$, where~$C_\activ$ is
defined for a given kernel in Lemma~\ref{lemma:activations}.
Then, the CNN function~$f_\activ$ defined above is in the RKHS~$\Hcal_{\mathcal{K}_n}$, with norm
\begin{equation}
\label{eq:spectral_bound}
\|f_{\sigma}\|^2 \leq \|w_{n+1}\|^2 ~ C_\sigma^2(\|W_n\|_2^2 ~ C_\sigma^2(\|W_{n-1}\|_2^2 \ldots C_\sigma^2(\|W_2\|_2^2~C_\sigma^2(\|W_1\|_F^2)) \ldots)).
\end{equation}
The norms are defined as follows:
\begin{align*}
\|W_k\|_2^2 &= \int_{S_k} \|W_k(u)\|_2^2 d \nu_k(u),~~ \text{ for }k = 2, \ldots, n \\
\|W_1\|_F^2 &= \int_{S_1} \|W_1(u)\|_F^2 d \nu_1(u),
\end{align*}
where~$W_k(u)$ is the matrix $(w_k^{ij}(u))_{ij}$, $\| \cdot \|_2$ the spectral norm, and $\| \cdot \|_F$ the Frobenius norm.
\end{proposition}
As an example, if we consider $\kappa_1 = \cdots = \kappa_n$ to be one of the kernels introduced in Section~\ref{sec:kernel_construction} and take $\sigma = \kappa_1$ so that~$C_\sigma^2(\lambda^2) = \kappa_1(\lambda^2)$,
then constraining the norms at each layer to be smaller than 1 ensures $\|f_\sigma\| \leq 1$, since for $\lambda \leq 1$ we have $C_\sigma^2(\lambda^2) \leq C_\sigma^2(1) = \kappa_1(1) = 1$.
If we consider linear kernels and $\sigma(u) = u$, we have $C_\sigma^2(\lambda^2) = \lambda^2$ and the bound becomes $\|f_\sigma\| \leq \|w_{n+1}\| \|W_n\|_2 \cdots \|W_2\|_2 \|W_1\|_F$.
If we ignore the convolutional structure (\ie, only taking 1x1 patches on a 1x1 image),
the norm involves a product of spectral norms at each layer (ignoring the first layer),
a quantity which also appears in recent generalization bounds~\citep{bartlett2017spectrally,neyshabur2017pac}.
While such quantities have proven useful to explain some generalization phenomena, such as the
behavior of networks trained on data with random labels~\citep{bartlett2017spectrally,zhang2016understanding},
some authors have pointed out that spectral norms may yield overly pessimistic generalization bounds
when comparing with simple parameter counting~\citep{arora2018stronger},
and our results may display similar drawbacks.
We note, however, that Proposition~\ref{prop:cnns_spectral} only gives an upper bound, and the actual RKHS
norm may be smaller in practice. It may also be that the norm is not well controlled during training,
and that the obtained bounds may not fully explain the generalization behavior observed in practice.
Using such quantities to regularize during training
may then yield bounds that are less vacuous~\citep{bietti2018regularization}.

\paragraph{Generalization and stability.}
The results of this section imply that our study of the geometry of the
kernel representations, and in particular the stability and invariance
properties of Section~\ref{sec:stability}, apply to the generic CNNs defined
above, thanks to the Lipschitz smoothness relation~\eqref{eq:cs}.
The smoothness is then controlled by the RKHS norm of these
functions, which sheds light on the links between generalization and stability.
In particular, functions with low RKHS norm provide better generalization guarantees
on unseen data, as shown by the margin bound in Eq.~\eqref{eq:margin_bound}.
This implies, for instance, that generalization is harder if the task requires classifying two
slightly deformed images with different labels,
since separating such predictions by some margin requires a function with large RKHS norm
according to our stability analysis.
In contrast, if a stable function (\ie, with small RKHS norm) is sufficient to do well on a training set, learning becomes ``easier'' and few samples may be enough for good generalization.

\subsection{Stability and Generalization with Generic Activations}
\label{sub:generic_activations}

Our study of stability and generalization so far has relied on kernel methods,
which allows us to separate learned models from data representations
in order to establish tight connections between the stability of representations
and statistical properties of learned CNNs through RKHS norms.
One important caveat, however, is that our study is limited to CNNs with
a class of smooth and homogeneous activations described in Section~\ref{sub:rkhs_activations},
which differ from generic activations used in practice such as ReLU or tanh.
Indeed, ReLU is homogeneous but lacks the required smoothness, while tanh is not homogeneous.
In this section, we show that our stability results can be extended to the predictions of CNNs with such
activations, and that stability is controlled by a quantity based on spectral norms,
which plays an important role in recent results on generalization.
This confirms a strong connection between stability and generalization in this more
general context as well.

\paragraph{Stability bound.}
We consider an activation function~$\sigma: \R \to \R$ that is $\rho$-Lipschitz and satisfies $\sigma(0) = 0$.
Examples include ReLU and tanh activations, for which $\rho = 1$.
The CNN construction is similar to Section~\ref{sub:cnns_rkhs}
with feature maps~$z_k$ in $L^2(\Omega,\R^{p_k})$, and a final prediction function~$f_\sigma$
defined with an inner product $\langle w_{n+1}, z_n \rangle$.
The only change is the non-linear mapping in Eq.~(\ref{eq:cnn}),
which is no longer homogeneized, and can be rewritten as
\begin{align*}
\tilde{z}_k(u) &= \varphi_k(P_k z_{\kmone}(u))
\defin \sigma\! \left( \int_{S_k} W_k(v) (P_k z_{\kmone}(u))(v) d \nu_k(v) \right),
\end{align*}
where $\sigma$ is applied component-wise. The non-linear mapping~$\varphi_k$ on patches 
satisfies
\[
\|\varphi_k(z) - \varphi_k(z')\| \leq \rho_k \|z - z'\| \quad \text{ and } \quad \|\varphi_k(z)\| \leq \rho_k \|z\|,
\]
where $\rho_k = \rho \|W_k\|_2$ and $\|W_k\|_2$ is the spectral norm of $W_k : L^2(S_k, \R^{p_\kmone}) \to \R^{p_k}$
defined~by
\[
\|W_k\|_2 = \sup_{\|z\| \leq 1} \left\|\int_{S_k} W_k(v) z(v) d \nu_k(v) \right\|.
\]
We note that this spectral norm is slightly different than the mixed norm used in Proposition~\ref{prop:cnns_spectral}.
By defining an operator~$M_k$ that applies~$\varphi_k$ pointwise as in Section~\ref{sec:kernel_construction},
the construction of the last feature map takes the same form as that of the multilayer kernel representation,
so that the results of Section~\ref{sec:stability} apply,
leading to the following stability bound on the final predictions:
\begin{equation}
\label{eq:generic_cnn_stability}
| f_\sigma(L_\tau x) - f_\sigma(x) | \leq \rho^n \|w_{n+1}\| \left(\prod_k \|W_k\|_2 \right) \left( C_1 \left(1 + n \right) \|\nabla \tau\|_\infty
		 + \frac{C_2}{\sigma_n} \|\tau\|_\infty \right) \|x\|.
\end{equation}

\paragraph{Link with generalization.}
The stability bound~\eqref{eq:generic_cnn_stability} takes a similar form to the one obtained for
CNNs in the RKHS, with the RKHS norm replaced by the product of spectral norms.
In contrast to the RKHS norm, such a quantity does not directly lead to generalization bounds;
however, a few recent works have provided meaningful generalization bounds for deep neural networks
that involve the product of spectral norms~\citep{bartlett2017spectrally,neyshabur2017pac}.
Thus, this suggests that stable CNNs have better generalization properties,
even when considering generic CNNs with ReLU or tanh activations.
Nevertheless, these bounds typically involve an additional factor consisting of other matrix norms
summed across layers, which may introduce some dependence on the number of parameters,
and do not directly support convolutional structure.
In contrast, our RKHS norm bound based on spectral norms given in Proposition~\ref{prop:cnns_spectral}
directly supports convolutional structure, and has no dependence on the number of parameters.

\section{Discussion and Concluding Remarks}
\label{sec:discussion_conclusion}

In this paper, we introduce a multilayer convolutional kernel representation (Section~\ref{sec:kernel_construction});
we show that it is stable to the action of diffeomorphisms, and that it can be made invariant to groups of transformations (Section~\ref{sec:stability});
and finally we explain connections between our representation and generic convolutional networks
by showing that certain classes
of CNNs with smooth activations are contained in the RKHS of the
full multilayer kernel (Section~\ref{sec:link_with_cnns}). A consequence of this last result is that
the stability results of Section~\ref{sec:stability} apply to any CNN function~$f$ from that class,
by using the relation
\begin{equation*}
|f(L_\tau x) - f(x)| \leq \|f\| \|\Phi_n(L_\tau x) - \Phi_n(x) \|,
\end{equation*}
which follows from~(\ref{eq:cs}), assuming a linear prediction layer.
In the case of CNNs with generic activations such as ReLU,
the kernel point of view is not applicable, and the separation between model and representation
is not as clear. However, we show in Section~\ref{sub:generic_activations} that a similar
stability bound can be obtained, with the product of spectral norms at each layer
playing a similar role to the RKHS norm of the CNN.
In both cases, a quantity that characterizes complexity of a model appears in the final
bound on predicted values --- either the RKHS norm or the product of spectral norms ---,
and this complexity measure is also closely related to generalization.
This implies that learning with stable CNNs is ``easier'' in terms of sample complexity,
and that the inductive bias of CNNs is thus suitable to tasks that present some invariance
under translation and small local deformation, as well as more general transformation groups,
when the architecture is appropriately constructed.

In order to ensure stability, the previous bounds suggest that one should control
the RKHS norm~$\|f\|$, or the product of spectral norms when using generic activations;
however, these quantities are difficult to control with standard approaches to learning CNNs,
such as backpropagation.
In contrast, traditional kernel methods typically control this norm by using it as an explicit regularizer
in the learning process, making such a stability guarantee more useful. In order to avoid the
scalability issues of such approaches, convolutional kernel networks
approximate the full kernel map~$\Phi_n$ by taking appropriate projections as explained in Section~\ref{subsec:ckn},
leading to a representation~$\tilde{\Phi}_n$ that can be represented with a practical
representation~$\psi_n$ that preserves the Hilbert space structure isometrically
(using the finite-dimensional descriptions of points in the RKHS given in~\eqref{eq:ckn_proj_repr}).
Section~\ref{sub:stability_with_approximation} shows that such representations
satisfy the same stability and invariance results as the full representation,
at the cost of losing information.
Then, if we consider a CKN function of the form~$f_w(x) = \langle w, \psi_n(x) \rangle$,
stability is obtained thanks to the relation
\begin{equation*}
|f_w(L_\tau x) - f_w(x)| \leq \|w\| \|\psi_n(L_\tau x) - \psi_n(x) \| = \|w\| \|\tilde{\Phi}_n(L_\tau x) - \tilde{\Phi}_n(x)\|.
\end{equation*}
In particular, learning such a function by controlling the norm of~$w$, \eg, with~$\ell_2$ regularization,
provides a natural way to explicitly control stability.
In the context of CNNs with generic activations, it has been
suggested \citep[see, \eg,~][]{zhang2016understanding} that optimization algorithms may play
an important role in controlling their generalization ability,
and it may be plausible that these impact the RKHS norm of a learned CNN, or its spectral norms.
A better understanding of such implicit regularization behavior would be interesting,
but falls beyond the scope of this paper.
Nevertheless, modern CNNs trained with SGD have been found to be highly unstable to small, additive
perturbations known as ``adversarial examples''~\citep{szegedy2013intriguing}, which suggests that
the RKHS norm of these models may be quite large, and that controlling it explicitly during learning
might be important to learn more stable models~\citep{bietti2018regularization,cisse2017parseval}.

\acks{This work was supported by a grant from ANR (MACARON project under grant
number ANR-14-CE23-0003-01), by the ERC grant number 714381 (SOLARIS project),
and from the MSR-Inria joint centre. }

\newpage
\appendix
\renewcommand{\theHsection}{A\arabic{section}}

\section{Useful Mathematical Tools} 
\label{sec:basic_tools}

In this section, we present preliminary mathematical tools that are used in our
analysis.
\paragraph{Harmonic analysis.}
We recall a classical result from harmonic analysis \citep[see, \eg,~][]{stein1993harmonic}, which was used many times by~\citet{mallat2012group} to prove
the stability of the scattering transform to the action of diffeomorphisms.
\begin{appxlemma}[Schur's test]\label{lemma:schur}
   Let $\Hcal$ be a Hilbert space and $\Omm$ a subset of $\Real^d$.
   Consider $T$ an integral operator with kernel $k: \Omm \times \Omm \to \Real$, meaning that for all~$u$ in~$\Omm$ and $x$ in~$L^2(\Omm,\Hcal)$,
   \begin{equation}
Tx(u) =
      \int_\Omm k(u,v) x(v) dv, \label{eq:schur}
   \end{equation}
   where the integral is a Bochner integral \citep[see,~][]{diestel,muandet2017kernel} when $\Hcal$ is infinite-dimensional.
   If 
   \begin{equation*}
\forall u \in \Omm, ~~~\int |k(u,v)| dv \leq C \quad \text{ and } \quad \forall v \in \Omm, ~~~\int |k(u,v)| du \leq C,
\end{equation*}
for some constant $C$,
then, $Tx$ is always in $L^2(\Omm,\Hcal)$ for all $x$ in $L^2(\Omm,\Hcal)$ and we have $\|T\| \leq C$.
\end{appxlemma}

Note that while the proofs of the lemma above are typically given for real-valued functions in~$L^2(\Omm, \R)$, the result can easily be extended to Hilbert space-valued functions $x$ in $L^2(\Omm, \Hc)$. In order to prove this, we consider the integral operator $|T|$ with kernel~$|k|$ that operates on $L^2(\Omm,\R_{+})$, meaning that $|T|$ is defined as in~(\ref{eq:schur}) by replacing $k(u,v)$ by the absolute value $|k(u,v)|$. Then, consider $x$ in $L^2(\Omm,\Hcal)$ and use the triangle inequality property of Bochner integrals:
\begin{align*}
\|Tx\|^2 = \int_\Omm \|Tx(u)\|^2 du \leq \int_\Omm \left( \int_\Omm |k(u,v)| \|x(v)\| dv \right)^2 du = \||T| |x| \|^2,
\end{align*}
where the function $|x|$ is such that $|x|(u) = \|x(u)\|$ and thus $|x|$ is in $L^2(\Omm,\R_{+})$. We may now apply Schur's test to the operator $|T|$ for real-valued functions, which gives $\||T|\| \leq C$.
Then, noting that $\||x|\| = \|x\|$, we conclude with the inequality $\|Tx\|^2 \leq \||T| |x| \|^2 \leq \||T|\|^2 \|x\|^2 \leq C^2 \|x\|^2$.

The following lemma shows that the pooling operators~$A_k$ defined in Section~\ref{sec:kernel_construction} are non-expansive.
\begin{appxlemma}[Non-expansiveness of pooling operators]\label{lemma:pooling}
If~$h(u) := (2 \pi)^{-d/2} \exp(-|u|^2/2)$, then the pooling operator~$A_\sigma$ defined for any~$\sigma > 0$ by
\begin{equation*}
A_\sigma x(u) = \int_{\R^d} \sigma^{-d} h\left(\frac{u - v}{\sigma}\right) x(v) dv,
\end{equation*}
has operator norm~$\|A_\sigma\| \leq 1$.
\end{appxlemma}
\begin{proof}
With the notations from above, we have $\|A_\sigma x\| \leq \| |A_\sigma| |x| \| = \| h_\sigma \ast |x| \|$,
where $h_\sigma := \sigma^{-d} h (\cdot / \sigma)$ and~$\ast$ denotes convolution.
By Young's inequality, we have $\| h_\sigma \ast |x| \| \leq \|h_\sigma\|_1 \cdot \||x|\| = 1 \cdot \||x|\| = \|x\|$,
which concludes the proof.
\end{proof}

\paragraph{Kernel methods.}
We now recall a classical result that characterizes the reproducing kernel
Hilbert space (RKHS) of functions defined from explicit Hilbert space mappings
\citep[see, \eg,~][\S 2.1]{saitoh1997integral}.
\begin{appxtheorem}
\label{thm:rkhs}
Let~$\psi: \Xcal \to H$ be a feature map to a Hilbert space~$H$, and let~$K(z, z') := \langle \psi(z), \psi(z') \rangle_H$ for~$z, z' \in \Xcal$. Let~$\Hcal$ be the linear subspace defined by
\begin{displaymath}
   \Hcal := \{ f_w ~;~ w \in H \} \st f_w: z \mapsto \langle w, \psi(z) \rangle_H,
\end{displaymath}
and consider the norm
\begin{displaymath}
   \|f_w\|_{\Hcal}^2 := \inf_{w'\in H} \{ \|w'\|_H^2  \st f_w  = f_{w'} \}.
\end{displaymath}
Then $\Hcal$ is the reproducing kernel Hilbert space associated to kernel~$K$.
\end{appxtheorem}

A consequence of this result is that the RKHS of the kernel~$\mathcal{K}_n(x, x') = \langle \Phi(x), \Phi(x') \rangle$,
defined from a given final representation~$\Phi(x) \in \Hcal_{n+1}$ such as the one introduced in Section~\ref{sec:kernel_construction},
contains functions of the form $f: x \mapsto \langle w, \Phi(x) \rangle$ with $w \in \Hcal_{n+1}$,
and the RKHS norm of such a function satisfies~$\|f\| \leq \|w\|_{\Hcal_{n+1}}$.


\section{Proofs Related to the Multilayer Kernel Construction} 
\label{sec:choices_of_kernels}

\subsection{Proof of Lemma~\ref{lemma:dp_kernels} and Non-Expansiveness of the Gaussian Kernel}
\label{sub:nonexp_proofs}
We begin with the proof of Lemma~\ref{lemma:dp_kernels} related to homogeneous dot-product kernels~\eqref{eq:dp_kernel_appx}.
\begin{proof}
   In this proof, we drop all indices $k$ since there is no ambiguity.
   We will prove the more general result that $\varphi$ is $\rho_k$-Lipschitz
   with $\rho_k = \max(1, \sqrt{\kappa'(1)})$ for any value of $\kappa'(1)$ (in particular,
   it is non-expansive when $\kappa'(1) \leq 1$).

   Let us consider the Maclaurin expansion $\kappa(u) =
   \sum_{j=0}^{+\infty} b_j u^j < +\infty$ with $b_j \geq 0$ for all $j$ and
   all $u$ in $[-1,+1]$. Recall that the condition $b_j \geq 0$ comes from the
   positive-definiteness of~$K$~\citep{schoenberg}.
   Then, we have $\kappa'(u) = \sum_{j=1}^{+\infty} j b_j u^{j-1}$.
   Noting that $j b_j u^{j-1} \leq j b_j$ for $u \in [-1, 1]$, we have $\kappa'(u) \leq \kappa'(1)$ on $[-1, 1]$.
   The fundamental theorem of calculus then yields, for $u \in [-1, 1]$,
   \begin{equation}
   \label{eq:kappa_inequality}
   \kappa(u) = \kappa(1) - \int_u^1 \kappa'(t) dt \geq \kappa(1) - \kappa'(1) (1 - u).
   \end{equation}
   Then, if $z, z' \neq 0$,
\begin{equation*}
\|\varphi(z) - \varphi(z')\|^2 = K(z,z) + K(z',z') - 2 K(z,z') = \|z\|^2 + \|z'\|^2 - 2\|z\| \|z'\| \kappa (u),
\end{equation*}
with $u= \langle z, z' \rangle / ( \|z\| \|z'\|)$. Using~\eqref{eq:kappa_inequality} with $\kappa(1) = 1$, we have
\begin{equation*}
   \begin{split}
      \|\varphi(z) - \varphi(z')\|^2 & \leq \|z\|^2 + \|z'\|^2 - 2\|z\| \|z'\|\left( 1-\kappa'(1) + \kappa'(1)u\right) \\ 
      & = (1-\kappa'(1)) \left( \|z\|^2 + \|z'\|^2 - 2\|z\| \|z'\| \right) + \kappa'(1)\left( \|z\|^2 + \|z'\|^2 - 2 \langle z, z' \rangle \right) \\
      & = (1-\kappa'(1)) \left| \|z\|- \|z'\|  \right|^2 + \kappa'(1)\|z-z'\|^2 \\
      &\leq \begin{cases}
         \|z - z'\|^2, &\text{ if }0 \leq \kappa'(1) \leq 1\\
         \kappa'(1) \|z - z'\|^2, &\text{ if }\kappa'(1) > 1
      \end{cases} \\
      & = \rho_k^2 \|z-z'\|^2,
   \end{split}
\end{equation*}
   with $\rho_k = \max(1, \sqrt{\kappa'(1)})$, which yields the desired result.
   Finally, we remark that we have shown the relation $\kappa(u) \geq \kappa(1)-\kappa'(1)
   + \kappa'(1)u$; when $\kappa'(1)=1$, this immediately
   yields~(\ref{eq:lower_linear}).

   If $z=0$ or $z' = 0$, the result also holds trivially. For example,
\begin{equation*}
\|\varphi(z) - \varphi(0)\|^2 = K(z,z) + K(0,0) - 2 K(z,0) = \|z\|^2 = \|z-0\|^2.
\end{equation*}
\end{proof}

   \paragraph{Non-expansiveness of the Gaussian kernel.}
We now consider the Gaussian kernel
\begin{equation*}
K(z, z') := e^{-\frac{\alpha}{2}\|z-z'\|^2},
\end{equation*}
with feature map~$\varphi$.
   We simply use the convexity inequality $e^u \geq 1 + u$ for all $u$, and 
\begin{equation*}
   \|\varphi(z) - \varphi(z')\|^2 = K(z,z) + K(z',z') - 2 K(z,z') = 2 - 2 e^{-\frac{\alpha}{2}\|z-z'\|^2} \leq \alpha \|z-z'\|^2.
\end{equation*}
In particular, $\varphi$ is non-expansive when~$\alpha \leq 1$.


\section{Proofs of Recovery and Stability Results} 
\label{sec:proofs}

\subsection{Proof of Lemma~\ref{lemma:signal_recovery}} 
\label{sub:proof_of_recovery}
\begin{proof}
We denote by~$\bar{\Omm}$ the discrete set of sampling points considered in this lemma.
The assumption on~$\bar{\Omm}$ can be written as $\{u + v ~;~ u \in \bar{\Omm}, v \in S_k\} = \Omega$.

Let~$B$ denote an orthonormal basis of the Hilbert space~$\Pcal_k = L^2(S_k,
\Hcal_{k-1})$, and define the linear function~$f_w$ in $\Hcal_k$ such that
   $f_w: z \mapsto \langle w, z \rangle $ for~$w$ in $\Pcal_k$. We thus have
\begin{align*}
P_k x_{\kmone}(u) &= \sum_{w \in B} \langle w, P_k x_{\kmone} (u) \rangle w \\
   &= \sum_{w \in B} f_w(P_k x_{\kmone}(u)) w \\
   &= \sum_{w \in B} \langle f_w, M_k P_k x_{\kmone}(u) \rangle w,
\end{align*}
using the reproducing property in the RKHS~$\Hcal_k$. Applying the pooling
operator~$A_k$ yields
\begin{align*}
A_k P_k x_{\kmone}(u) &= \sum_{w \in B} \langle f_w, A_k M_k P_k x_{\kmone}(u)\rangle w, \\
   &= \sum_{w \in B} \langle f_w, x_{k}(u)\rangle w.
\end{align*}
   Noting that $A_k P_k x_{\kmone} = A_k (L_v x_{\kmone})_{v \in S_k} = (A_k L_v x_{\kmone})_{v\in S_k} = 
   (L_v A_k x_{\kmone})_{v \in S_k} = P_k A_k x_{\kmone}$, with $L_v x_{\kmone}(u) := x_{\kmone}(u + v)$, we can
choose~$v$ in $S_k$ and obtain from the previous relations
\begin{align*}
A_k x_{\kmone}(u + v) = \sum_{w \in B} \langle f_w, x_{k}(u)\rangle w(v).
\end{align*}
Thus, taking all sampling points~$u \in \bar{\Omm}$ and all~$v \in S_k$, we
have a full view of the signal~$A_k x_{\kmone}$ on all of~$\Omega$ by our assumption
on the set~$\bar{\Omm}$.

For~$f \in \Hcal_{\kmone}$, the signal~$\langle
f, x_{\kmone}(u) \rangle$ can then be recovered by deconvolution as follows:
\begin{align*}
\langle f, x_{\kmone}(u) \rangle
   = \mathcal{F}^{-1} \left( \frac{\mathcal{F}(\langle f, A_k x_{\kmone}
   (\cdot)\rangle)}{\mathcal{F}(h_{\sigma_k})} \right)(u),
\end{align*}
where~$\mathcal{F}$ denotes the Fourier transform. Note that the inverse Fourier
transform is well-defined here because the signal~$\langle f, A_k x_k(\cdot)\rangle$ is
itself a convolution with~$h_{\sigma_k}$, and~$\mathcal{F}(h_{\sigma_k})$ is strictly
positive as the Fourier
transform of a Gaussian is also a Gaussian.

By considering all elements~$f$ in an orthonormal basis of~$\Hcal_{\kmone}$, we
can recover~$x_{\kmone}$. The map $x_k$ can then be reconstructed trivially by applying
   operators~$P_k$, $M_k$ and~$A_k$ on $x_{\kmone}$.

\end{proof}


\subsection{Proof of Lemma~\ref{lemma:discretization}} 
\label{sub:proof_of_discretization}
\begin{proof}
   In this proof, we drop the bar notation on all quantities for simplicity;
   there is indeed no ambiguity since all signals are discrete here.  First, we
   recall that~$\Hcal_k$ contains all linear functions on~$\Pcal_k =
   \Hcal_{\kmone}^{e_k}$; thus, we may consider in particular functions $f_{j,
   w}(z) := e_k^{1/2}\langle w, z_j \rangle$ for $j \in \{1, \ldots, e_k\}$,~$w \in
   \Hcal_{\kmone}$, and $z=(z_1,z_2,\ldots,z_{e_k})$ in $\Pcal_k$. 
   Then, we may evaluate
\begin{align*}
\langle f_{j,w}, s_k^{-1/2} x_k[n] \rangle
	&= \sum_{m \in \Z} {h}_k[n s_k - m] \langle f_{j,w}, {M}_k {P}_k x_{\kmone}[m] \rangle \\
        &= \sum_{m \in \Z} {h}_k[n s_k - m] \langle f_{j,w}, \varphi_k({P}_k x_{\kmone}[m] ) \rangle \\
	&= \sum_{m \in \Z} {h}_k[n s_k - m] f_{j,w}({P}_k x_{\kmone}[m]) \\
	&= \sum_{m \in \Z} {h}_k[n s_k - m] \langle w, x_{\kmone}[m + j] \rangle \\
	&= \sum_{m \in \Z} {h}_k[n s_k + j - m] \langle w, x_{\kmone}[m] \rangle \\
        &= (h_k \ast \langle w, x_{\kmone} \rangle ) [n s_k + j],
\end{align*}
where, with an abuse of notation, $\langle w, x_{\kmone} \rangle$ is the
real-valued discrete signal such that $\langle w, x_{\kmone} \rangle [n] = \langle w, x_{\kmone}[n] \rangle$.
Since integers of the form $(n s_k + j)$ cover all of~$\Z$ according to the assumption $e_k \geq s_k$, we have a full
view of the signal $(h_k \ast \langle w, x_{\kmone} \rangle )$ on $\Z$.
We will now follow the same reasoning as in the proof of Lemma~\ref{lemma:signal_recovery} to recover $\langle w, x_{\kmone} \rangle$:
\begin{align*}
\langle w, x_{k-1} \rangle = \mathcal{F}^{-1} \left( \frac{\mathcal{F}(h_k \ast \langle w, x_{\kmone} \rangle )}{\mathcal{F}({h}_{k})} \right),
\end{align*}
where~$\mathcal{F}$ is the Fourier transform. Since the signals involved there are discrete, their Fourier transform are periodic with period $2\pi$,
and we note that $\mathcal{F}({h}_{k})$ is strictly positive and bounded away from zero. The signal~$x_{\kmone}$ is then recovered exactly
   as in the proof of Lemma~\ref{lemma:signal_recovery} by considering for~$w$ the elements of an orthonormal basis of $\Hcal_{\kmone}$.
\end{proof}

\subsection{Proof of Proposition~\ref{prop:general_bound}} 
\begin{proof}
Define $(MPA)_{k:j} := M_k P_{k} A_{k-1} M_{k-1}  P_{k-1} A_{k-2} \cdots M_j  P_j A_{j-1}$. Using the fact that $\|A_k\| \leq 1$, $\|P_k\| = 1$ and $M_k$ is non-expansive, we obtain
\begin{align*}
\| \Phi_n(L_\tau x) - \Phi_n(x) \| &= \|A_n (MPA)_{n:2} M_1 P_1 A_0 L_\tau x - A_n (MPA)_{n:2} M_1 P_1 A_0 x \| \\
	&\leq \|A_n (MPA)_{n:2} M_1 P_1 A_0 L_\tau x - A_n (MPA)_{n:2} M_1 L_\tau P_1 A_0 x \| \\
		&\quad + \|A_n (MPA)_{n:2} M_1 L_\tau P_1 A_0 x - A_n (MPA)_{n:2} M_1 P_1 A_0 x \| \\
	&\leq  \|[P_1 A_0, L_\tau]\| \|x\| \\
		&\quad + \|A_n (MPA)_{n:2} M_1 L_\tau P_1 A_0 x - A_n (MPA)_{n:2} M_1 P_1 A_0 x \|.
\end{align*}
Note that $M_1$ is defined point-wise, and thus commutes with $L_\tau$:
\[
   M_1 L_\tau x (u) = \varphi_1(L_\tau x(u)) = \varphi_1(x(u - \tau(u)) = M_1 x(u - \tau(u)) = L_\tau M_1 x (u).
\]
By noticing that $\|M_1 P_1 A_0 x\| \leq \|x\|$, we can expand the second term above in the same way. Repeating this by induction yields
\begin{align*}
\| \Phi_n(L_\tau x) - \Phi_n(x) \| &\leq \sum_{k=1}^n \|[P_k A_{k-1}, L_\tau]\| \|x\|  + \|A_n L_\tau (MPA)_{n:1} x - A_n (MPA)_{n:1} x \| \\
	&\leq \sum_{k=1}^n \|[P_k A_{k-1}, L_\tau]\| \|x\| + \|A_n L_\tau - A_n \| \|x\|,
\end{align*}
and the result follows by decomposing $A_n L_\tau = [A_n, L_\tau] + L_\tau A_n$ and applying the triangle inequality.
\end{proof}

\subsection{Proof of Lemma~\ref{lemma:stability}} 
\label{sub:stability_proof}

\begin{proof}
   The proof follows in large parts the methodology introduced by~\citet{mallat2012group} in the analysis of the stability of the scattering transform. More precisely, we will follow in part the proof of Lemma E.1 of~\citet{mallat2012group}.
   The kernel (in the sense of Lemma~\ref{lemma:schur}) of $A_\sigma$ is $h_\sigma(z - u) = \sigma^{-d} h(\frac{z - u}{\sigma})$.
Throughout the proof, we will use the following bounds on the decay of~$h$ for simplicity,
as in~\citet{mallat2012group}:\footnote{Note that a more precise analysis may be obtained by using finer decay bounds.}
\begin{align*}
|h(u)| &\leq \frac{C_h}{(1 + |u|)^{d+2}} \\
|\nabla h(u)| &\leq \frac{C'_h}{(1 + |u|)^{d+2}},
\end{align*}
which are satisfied for the Gaussian function~$h$ thanks to its exponential decay.

We now decompose the commutator
\begin{align*}
[L_c A_\sigma, L_\tau] = L_c A_\sigma L_\tau - L_\tau L_c A_\sigma = L_c (A_\sigma - L_c^{-1} L_\tau L_c A_\sigma L_\tau^{-1}) L_\tau = L_c T L_\tau,
\end{align*}
with $T := A_\sigma - L_c^{-1} L_\tau L_c A_\sigma L_\tau^{-1}$. Hence,
\begin{align*}
\|[L_c A_\sigma, L_\tau]\| \leq \|L_c\| \|L_\tau\| \|T\|.
\end{align*}
We have $\|L_c\| = 1$ since the translation operator $L_c$ preserves the norm. Note that we have
\begin{equation}
\label{eq:det_diffeo}
2^{-d} \leq (1 - \| \nabla \tau\|_\infty)^d \leq \det(I - \nabla \tau(u)) \leq (1 + \| \nabla \tau\|_\infty)^d \leq 2^d,
\end{equation}
for all~$u \in \Omega$. Thus, for~$f \in L^2(\Omega)$,
\begin{align*}
\|L_\tau f\|^2 &= \int_\Omega |f(z - \tau(z))|^2 dz = \int_\Omega |f(u)|^2 \det(I - \nabla \tau(u))^{-1} du \\
	&\leq (1 - \|\nabla \tau\|_\infty)^{-d} \|f\|^2,
\end{align*}
such that $\|L_\tau\| \leq (1 - \|\nabla \tau\|_\infty)^{-d/2} \leq 2^{d/2}$. This yields
\begin{align*}
\|[L_c A_\sigma, L_\tau]\| \leq 2^{d/2} \|T\|.
\end{align*}

\paragraph{Kernel of $T$.}
We now show that~$T$ is an integral operator and describe its kernel.
Let $\xi = (I - \tau)^{-1}$, so that $L_\tau^{-1} f(z) = f(\xi(z))$ for any function~$f$ in~$L^2(\Omega)$. We have
\begin{align*}
A_\sigma L_\tau^{-1} f(z) &= \int h_\sigma(z - v) f(\xi(v)) dv \\
	&= \int h_\sigma(z - u + \tau(u)) f(u) \det(I - \nabla \tau(u)), du
\end{align*}
using the change of variable $v = u - \tau(u)$, giving $\left|\frac{dv}{du}\right| = \det(I - \nabla \tau(u))$. Then note that $L_c^{-1} L_\tau L_c f(z) = L_\tau L_c f(z + c) = L_c f(z+c - \tau(z + c)) = f(z - \tau(z + c))$. This yields the following kernel for the operator $T$:
\begin{equation}
   k(z, u) = h_\sigma(z - u) - h_\sigma(z - \tau(z + c) - u + \tau(u)) \det(I - \nabla \tau(u)). \label{eq:kernelmallat}
\end{equation}
   A similar operator appears in Lemma E.1 of~\citet{mallat2012group}, whose kernel is identical to~(\ref{eq:kernelmallat}) when $c=0$.

   Like~\citet{mallat2012group}, we decompose $T = T_1 + T_2$, with kernels
\begin{align*}
k_1(z, u) &= h_\sigma(z - u) - h_\sigma((I - \nabla \tau(u))(z - u)) \det(I - \nabla \tau(u)) \\
k_2(z, u) &= \det(I - \nabla \tau(u)) \left(h_\sigma((I - \nabla \tau(u))(z - u)) - h_\sigma(z - \tau(z + c) - u + \tau(u)) \right).
\end{align*}
   The kernel $k_1(z,u)$ appears in~\citep{mallat2012group}, whereas the kernel
   $k_2(z,u)$ involves a shift $c$ which is not present in~\citep{mallat2012group}.
   For completeness, we include the proof of the bound for both operators, even
   though only dealing with $k_2$ requires slightly new developments.

\paragraph{Bound on $\|T_1\|$.}
   We can write~$k_1(z, u) = \sigma^{-d} g(u, (z - u)/\sigma)$ with
\begin{align*}
g(u, v) &= h(v) - h((I - \nabla \tau(u)) v) \det(I - \nabla \tau(u)) \\
	&= (1 - \det(I - \nabla \tau(u))) h((I - \nabla \tau(u)) v) + h(v) - h((I - \nabla \tau(u)) v).
\end{align*}
Using the fundamental theorem of calculus on~$h$, we have
\begin{align*}
h(v) - h((I - \nabla \tau(u))v) &= \int_0^1 \langle \nabla h ((I + (t - 1) \nabla \tau(u))v), \nabla \tau(u) v \rangle dt.
\end{align*}
Noticing that
\[
|(I + (t - 1) \nabla \tau(u))v| \geq (1 - \|\nabla \tau\|_\infty) |v| \geq (1/2) |v|,
\]
and that $\det(I - \nabla \tau(u))) \geq (1 - \|\nabla \tau\|_\infty)^d \geq 1 - d\|\nabla \tau\|_\infty$, we bound each term as follows
\begin{align*}
|(1 - \det(I - \nabla \tau(u))) h((I - \nabla \tau(u)) v)|
	&\leq d \|\nabla \tau\|_\infty \frac{C_h}{(1 + \frac{1}{2}|v|)^{d+2}} \\
\left| \int_0^1 \langle \nabla h ((I + (t - 1) \nabla \tau(u))v), \nabla \tau(u) v \rangle dt \right|
	&\leq \|\nabla \tau\|_\infty \frac{C'_h |v|}{(1 + \frac{1}{2}|v|)^{d + 2}}.
\end{align*}
We thus have
\begin{equation*}
|g(u, v)| \leq \| \nabla \tau \|_\infty \frac{C_h d + C'_h |v|}{(1 + \frac{1}{2}|v|)^{d + 2}}.
\end{equation*}
Using appropriate changes of variables in order to bound $\int |k_1(z, u)| du$ and $\int |k_1(z, u)| dz$, Schur's test yields
\begin{equation}
\label{eq:k1_norm}
\|T_1\| \leq C_1 \| \nabla \tau \|_\infty,
\end{equation}
with
\begin{equation*}
C_1 = \int_\Omega \frac{C_h d + C'_h |v|}{(1 + \frac{1}{2}|v|)^{d + 2}} dv
\end{equation*}

\paragraph{Bound on $\|T_2\|$.}
Let $\alpha(z, u) = \tau(z + c) - \tau(u) - \nabla \tau(u) (z - u)$, and note that we have
\begin{align}
\label{eq:alpha_bound}
|\alpha(z, u)| &\leq  |\tau(z + c) - \tau(u)| + |\nabla \tau(u) (z - u)| \nonumber\\
	&\leq \|\nabla \tau\|_\infty |z + c - u| + \|\nabla \tau\|_\infty |z - u| \nonumber\\
	&\leq \|\nabla \tau\|_\infty (|c| + 2|z - u|).
\end{align}
The fundamental theorem of calculus yields
\begin{align*}
k_2(z, u) = -\det(I - \nabla \tau(u)) \int_0^1 \langle \nabla h_\sigma(z - \tau(z+c) - u + \tau(u) - t \alpha(z, u)), \alpha(z, u) \rangle dt.
\end{align*}
We note that $|\det(I - \nabla \tau(u))| \leq 2^d$, and $\nabla h_\sigma(v) = \sigma^{-d-1} \nabla h(v/\sigma)$. Using the change of variable $z' = (z - u)/\sigma$, we obtain
\begin{align*}
\int & |k_2(z, u)|dz \\
& \leq 2^d \int \int_0^1 \left|\nabla h \left(z' + \frac{\tau(u + \sigma z' + c) - \tau(u) -t \alpha(u + \sigma z', u)}{\sigma}\right) \right| \left|\frac{\alpha(u + \sigma z', u)}{\sigma}\right| dt dz'.
\end{align*}
We can use the upper bound~\eqref{eq:alpha_bound}, together with our assumption $|c| \leq \kappa \sigma$:
\begin{equation}
\label{eq:alpha_normalized_bound}
\left|\frac{\alpha(u + \sigma z', u)}{\sigma}\right| \leq \|\nabla \tau\|_\infty (\kappa + 2|z'|).
\end{equation}
Separately, we have $|\nabla h(v(z'))| \leq C'_h / (1 + |v(z')|)^{d + 2}$, with
\begin{equation*}
v(z') := z' + \frac{\tau(u + \sigma z' + c) - \tau(u) -t \alpha(u + \sigma z', u)}{\sigma}.
\end{equation*}
For $|z'| > 2 \kappa$, we have
\begin{align*}
\left| \frac{\tau(u + \sigma z' + c) - \tau(u) -t \alpha(u + \sigma z', u)}{\sigma} \right|
	&= \left| t \nabla \tau(u) z' + (1-t)\frac{\tau(u + \sigma z' + c) - \tau(u)}{\sigma} \right| \\
	&\leq t \|\nabla \tau\|_\infty |z'| + (1 - t) \|\nabla \tau\|_\infty (|z'| + \kappa) \\
	&\leq \frac{3}{2} \|\nabla \tau\|_\infty |z'| \leq \frac{3}{4} |z'|,
\end{align*}
and hence, using the reverse triangle inequality, $|v(z')| \geq |z'| - \frac{3}{4} |z'| = \frac{1}{4}|z'|$. This yields the upper bound
\begin{equation}
\label{eq:grad_h_bound}
|\nabla h(v(z'))| \leq \begin{cases}
	C'_h, &\text{ if }|z'| \leq 2\kappa\\
	\frac{C'_h }{(1 + \frac{1}{4}|z'|)^{d + 2}}, &\text{ if }|z'| > 2\kappa.
\end{cases}
\end{equation}
Combining these two bounds, we obtain
\begin{align*}
\int |k_2(z, u)|dz \leq C_2 \|\nabla \tau\|_\infty,
\end{align*}
with
\begin{equation*}
C_2 := 2^d C'_h \left( \int_{|z'| < 2\kappa} (\kappa + 2|z'|)dz' + \int_{|z'| > 2\kappa} \frac{\kappa + 2|z'|}{(1 + \frac{1}{4}|z'|)^{d + 2}} dz' \right).
\end{equation*}
Note that the dependence of the first integral on~$\kappa$ is of order~$k^{d+1}$.
Following the same steps with the change of variable $u' = (z - u)/\sigma$, we obtain the bound $\int |k_2(z, u)|du \leq C_2 \|\nabla \tau\|_\infty$. Schur's test then yields
\begin{equation}
\label{eq:k2_norm}
\|T_2\| \leq C_2 \|\nabla \tau\|_\infty.
\end{equation}

We have thus proven
\begin{equation*}
\|[L_c A_\sigma, L_\tau]\| \leq 2^{d/2} \|T\| \leq 2^{d/2}(C_1 + C_2) \|\nabla \tau\|_\infty.
\end{equation*}
\end{proof}

\subsection{Discussion and Proof of Norm Preservation}
\label{sub:norm_preservation}
We now state a result which shows that while the kernel representation may lose some of the energy of the original
signal, it preserves a part of it, ensuring that the stability bound in Theorem~\ref{thm:stability} is non-trivial.
We consider in this section the full kernel representation, including a prediction layer, which is given by
$\Phi(x) = \varphi_{n+1}(\Phi_n(x))$, where~$\varphi_{n+1}$ is the kernel feature map of either
a Gaussian kernel~\eqref{eq:gauss_prediction_kernel} with~$\alpha = 1$, or a linear kernel~\eqref{eq:linear_prediction_kernel}.
In both cases, $\varphi_{n+1}$ is non-expansive, which yields
\begin{equation*}
\| \Phi(L_\tau x) - \Phi(x) \| \leq \| \Phi_n(L_\tau x) - \Phi_n(x) \|,
\end{equation*}
such that the stability result of Theorem~\ref{thm:stability} also applies to~$\Phi$.
For the Gaussian case, we trivially have a representation with norm~$1$,
which trivially shows a preservation of norm, while for the linear case,
at least part of the signal energy is preserved, in particular the energy in the low frequencies,
which is predominant, for instance, in natural images~\citep{torralba2003statistics}.

\begin{lemma}[Norm preservation]
\label{lemma:norm_lower_bound}
For the two choices of prediction layers, $\Phi(x)$ satisfies
\begin{equation*}
\|\Phi(x)\| = 1 \quad \text{(Gaussian),} \qquad
\|\Phi(x)\| \geq \|A_n A_{\nmone} \ldots A_0 x\| \quad \text{(Linear)}.
\end{equation*}
It follows that the representation~$\Phi$ is not contractive:
    \begin{equation}
       \sup_{x,x' \in L^2(\Omega,\Hcal_0)} \frac{\|\Phi(x) - \Phi(x')\|}{\|x-x'\|} = 1. \label{eq:non_contractive}
    \end{equation}
\end{lemma}

\begin{proof}
We begin by studying~$\|\Phi(x)\|$.
The Gaussian case is trivial since the Gaussian kernel mapping~$\varphi_{n+1}$ maps all points to the sphere.
In the linear case, we have
\begin{align*}
\|\Phi(x)\|^2 &= \|\Phi_n(x)\|^2 = \|A_n M_n P_n x_{\nmone}\|^2 \\
   &= \int \|A_n M_n P_n x_{\nmone}(u)\|^2 du \\
   &= \int \langle \int h_{\sigma_n}(u - v) M_n P_n x_{\nmone}(v) dv, \int h_{\sigma_n}(u - v') M_n P_n x_{\nmone}(v') dv' \rangle du \\
   &= \int \int \int h_{\sigma_n}(u - v) h_{\sigma_n}(u - v') \langle \varphi_n(P_n x_{\nmone}(v)), \varphi_n(P_n x_{\nmone}(v')) \rangle dv dv' du \\
   &\geq \int \int \int h_{\sigma_n}(u - v) h_{\sigma_n}(u - v') \langle P_n x_{\nmone}(v), P_n x_{\nmone}(v') \rangle dv dv' du \\
   &= \int \|A_n P_n x_{\nmone}(u)\|^2 du = \|A_n P_n x_{\nmone}\|^2,
\end{align*}
where the inequality follows from~$\langle \varphi_n(z), \varphi_n(z') \rangle = K_n(z, z')
\geq \langle z, z' \rangle$ (see Lemma~\ref{lemma:dp_kernels}).
Using Fubini's theorem and the fact that~$A_n$ commutes with translations, we have
\begin{align*}
\|A_n P_n x_{\nmone}\|^2 &= \int_{S_n} \|A_n L_v x_{\nmone}\|^2 d\nu_n(v) \\
 &= \int_{S_n} \|L_v A_n x_{\kmone}\|^2 d\nu_n(v) \\
 &= \int_{S_n} \|A_n x_{\kmone}\|^2 d\nu_n(v) \\
 &= \|A_n x_{\nmone}\|^2,
\end{align*}
where we used the fact that translations~$L_v$ preserve the norm.
Note that we have
\begin{align*}
A_n x_{\nmone} = A_n A_{\nmone} M_{\nmone} P_{\nmone} x_{n-2} = A_{n,\nmone} M_{\nmone} P_{\nmone} x_{n-2},
\end{align*}
where $A_{n,\nmone}$ is an integral operator with positive kernel~$h_{\sigma_n} \ast h_{\sigma_{\nmone}}$. Repeating the above relation then yields
\begin{align*}
\|\Phi(x)\|^2 \geq \|A_n x_{\nmone} \|^2 \geq \|A_n A_{\nmone} x_{n-1} \|^2
   \geq \ldots \geq \|A_n A_{\nmone} \ldots A_0 x\|^2,
\end{align*}
and the result follows.

We now show~\eqref{eq:non_contractive}.
By our assumptions on~$\varphi_{n+1}$ and on the operators~$A_k, M_k, P_k$, we have that~$\Phi$ is non-expansive,
so that
 \begin{displaymath}
    \sup_{x,x' \in L^2(\Omega,\Hcal_0)} \frac{\|\Phi(x) - \Phi(x')\|}{\|x-x'\|} \leq 1.
 \end{displaymath}
It then suffices to show that one can find $x, x'$ such that
the norm ratio $\frac{\|\Phi(x) - \Phi(x')\|}{\|x-x'\|}$ is arbitrarily close to 1.
In particular, we begin by showing that for any signal~$x \ne 0$ we have
 \begin{equation}
 \label{eq:phi_n_limlambda}
    \lim_{\lambda \to 1} \frac{\|\Phi(\lambda x) - \Phi(x)\|}{\|\lambda x-x\|} \geq \frac{\|A_\sigma x\|}{\|x\|},
 \end{equation}
where~$A_\sigma$ is the pooling operator with scale~$\sigma = (\sigma_n^2 +  \sigma_{\nmone}^2 + \ldots + \sigma_1^2)^{1/2}$,
and the result will follow by considering appropriate signals~$x$ that make this lower bound arbitrarily close to~1.

Note that by homogeneity of the kernels maps~$\varphi_k$ (which follows from the homogeneity of kernels~$K_k$),
and by linearity of the operators~$A_k$ and~$P_k$, we have~$\Phi_n(\lambda x) = \lambda \Phi_n(x)$ for
any~$\lambda \geq 0$. Taking~$\lambda > 0$, we have
\begin{align*}
\|\Phi_n(\lambda x) - \Phi_n(x)\| = (\lambda - 1) \|\Phi_n(x)\| \geq (\lambda - 1) \|A_n A_{\nmone} \ldots A_0 x\| = (\lambda - 1) \|A_\sigma x\|,
\end{align*}
adapting Lemma~\ref{lemma:norm_lower_bound} to the representation~$\Phi_n$.
Thus,
 \begin{displaymath}
    \lim_{\lambda \to 1} \frac{\|\Phi_n(\lambda x) - \Phi_n(x)\|}{\|\lambda x-x\|} \geq \frac{\|A_\sigma x\|}{\|x\|}.
 \end{displaymath}
When~$\varphi_{n+1}$ is linear, we immediately obtain~\eqref{eq:phi_n_limlambda} since $\|\Phi(\lambda x) - \Phi(x)\| = \|\Phi_n(\lambda x) - \Phi_n(x)\|$. For the Gaussian case, we have
\begin{align*}
\|\Phi(\lambda x) - \Phi(x)\|^2 &= 2 - 2 e^{-\frac{1}{2}\|\Phi_n(\lambda x) - \Phi_n(x)\|^2} \\
   &= 2 - 2 e^{-\frac{1}{2}(\lambda - 1)^2 \|\Phi_n(x)\|^2} \\
   &= (\lambda - 1)^2 \|\Phi_n(x)\|^2 + o((\lambda - 1)^2) \\
   &= \|\Phi_n(\lambda x) - \Phi_n(x)\|^2 + o((\lambda - 1)^2),
\end{align*}
which yields~\eqref{eq:phi_n_limlambda}.

By considering a Gaussian signal with scale~$\tau \gg \sigma$, we can make~$\frac{\|A_\sigma x\|}{\|x\|}$
arbitrarily close to~$1$ by taking an arbitrarily large~$\tau$. It follows that
\begin{align*}
\sup_x \lim_{\lambda \to 1} \frac{\|\Phi(\lambda x) - \Phi(x)\|}{\|\lambda x - x\|} = 1,
\end{align*}
which yields the result.
\end{proof}

\subsection{Proof of Lemma~\ref{lemma:roto_stability}}
\begin{proof}
We have
\begin{align*}
P x((u, \eta)) &= (v \in \tilde S \mapsto x((u, \eta)(v, 0))) \\
   &= (v \in \tilde S \mapsto x((u + R_\eta v, \eta))) \\
   &= (v \in R_\eta \tilde S \mapsto x((u + v, \eta))) \\
A x((u, \eta)) &= \int_{\R^2} x((u, \eta) (v, 0)) h(v) dv \\
   &= \int_{\R^2} x((u + R_\eta v, \eta)) h(v) dv \\
   &= \int_{\R^2} x((v, \eta)) h(R_{-\eta} (v - u)) dv \\
   &= \int_{\R^2} x((v, \eta)) h(u - v) dv,
\end{align*}
where the last equality uses the circular symmetry of a Gaussian around the origin.
For a diffeomorphism~$\tau$, we denote by $L_\tau$ the action operator given by
$L_\tau x((u, \eta)) = x((\tau(u), 0)^{-1}(u, \eta)) = x((u - \tau(u), \eta))$.
If we denote $x(\cdot, \eta)$ the $L^2({\R^2})$ signal obtained from a signal $x \in L^2(G)$ at a fixed angle,
we have shown
\begin{align*}
(Px)(\cdot, \eta) &= \tilde P_\eta(x(\cdot, \eta)) \\
(Ax)(\cdot, \eta) &= \tilde A (x(\cdot, \eta)) \\
(L_\tau x)(\cdot, \eta) &= \tilde L_\tau (x(\cdot, \eta)),
\end{align*}
where $\tilde P_\eta, \tilde A, \tilde L_\tau$ are defined on $L^2({\R^2})$ as in Section~\ref{sec:kernel_construction},
with a rotated patch $R_\eta \tilde S$ for $\tilde P_\eta$.
Then, we have, for a signal $x \in L^2(G)$,
\begin{align*}
\|[P A, L_\tau] x\|_{L^2(G)}^2 &= \int \|([P A, L_\tau] x) (\cdot, \eta)\|_{L^2({\R^2})}^2 d \mu_c(\eta) \\
   &= \int \|[\tilde P_\eta \tilde A, \tilde L_\tau] (x (\cdot, \eta))\|_{L^2({\R^2})}^2 d \mu_c(\eta) \\
   &\leq \int \|[\tilde P_\eta \tilde A, \tilde L_\tau] \|^2 \|x(\cdot, \eta)\|_{L^2({\R^2})}^2 d \mu_c(\eta) \\
   &\leq \left( \sup_\eta \|[\tilde P_\eta \tilde A, \tilde L_\tau] \|^2 \right) \|x\|_{L^2(G)}^2,
\end{align*}
so that $\|[P A, L_\tau]\|_{L^2(G)} \leq \sup_\eta \|[\tilde P_\eta \tilde A, \tilde L_\tau] \|_{L^2({\R^2})}^2$.
Note that we have $\sup_{c \in R_{\eta} \tilde S} |c| = \sup_{c \in \tilde S} |c| \leq \kappa \sigma$,
since rotations preserve the norm,
so that we can bound each $\|[\tilde P_\eta \tilde A, \tilde L_\tau]\|$ as in Section~\ref{sub:stability_results}
to obtain the desired result.
Similarly, $\|L_\tau A - A\|$ can be bounded as in Section~\ref{sub:stability_results}.

\end{proof}

\subsection{Proof of Theorem~\ref{thm:roto_stability}}
\begin{proof}
First, note that $A_c$ can be written as an integral operator
\[
A_c x(u) = \int x((v, \eta)) k(u, (v, \eta)) d \mu((v, \eta)),
\]
with $k(u, (v, \eta)) = \delta_u(v)$, where $\delta$ denotes the Dirac delta function.
We have
\[
\int |k(u, (v, \eta))| d \mu((v, \eta)) = \int |k(u, (v, \eta))| du = 1.
\]
By Schur's test, we thus obtain $\|A_c\| \leq 1$.
Then, note that $(\tau(u), \theta) = (0, \theta) (R_{-\theta} \tau(u), 0)$,
so that $L_{(\tau, \theta)} = L_{(0, \theta)} L_{\tau_\theta}$,
where we write $\tau_\theta(u) = R_{-\theta} \tau(u)$.
Additionally, it is easy to see that $A_c L_{(0, \theta)} = A_c$. We have
\begin{align*}
\| A_c \Phi_n(L_{(\tau, \theta)} x) - A_c \Phi_n(x) \|
   &= \| A_c \Phi_n(L_{(0, \theta)} L_{\tau_\theta} x) - A_c \Phi_n(x) \| \\
   &= \| A_c L_{(0, \theta)} \Phi_n(L_{\tau_\theta} x) - A_c \Phi_n(x) \| \\
   &= \| A_c \Phi_n(L_{\tau_\theta} x) - A_c \Phi_n(x) \| \\
   &\leq \|\Phi_n(L_{\tau_\theta} x) - \Phi_n(x) \|,
\end{align*}
using the fact that the representation $\Phi_n$ is equivariant to roto-translations by construction.

We conclude by using Lemma~\ref{lemma:roto_stability} together with an adapted version of Proposition~\ref{prop:general_bound},
and by noticing that $\|\nabla \tau_\theta\|_\infty = \|\nabla \tau\|_\infty$ and $\|\tau_\theta\|_\infty = \|\tau\|_\infty$.
\end{proof}


\section{Proofs Related to the Construction of CNNs in the RKHS} 
\label{sec:kernel_space_description}
\subsection{Proof of Lemma~\ref{lemma:activations}}
\begin{proof}
   Here, we drop all indices~$k$ since there is no ambiguity.
We will now characterize the functional space~$\Hcal$ by following the same strategy
as \citet{zhang2016l1,zhang2016convexified} for the non-homogeneous Gaussian and
inverse polynomial kernels on Euclidean spaces.
Using the Maclaurin expansion of~$\kappa$, we can define the following explicit
   feature map for the dot-product kernel $K_{\text{dp}}(z, z') := \kappa(\langle z, z' \rangle)$,
   for any $z$ in the unit-ball of~$\Pcal$:
\begin{equation}
   \begin{split}
      \psi_{\text{dp}}(z) & = \left(\sqrt{b_0}, \sqrt{b_1} z, \sqrt{b_2} z \otimes z,  \sqrt{b_3}  z \otimes z \otimes z, \ldots \right) \\
              & = \left(\sqrt{b_j} z^{\otimes j}   \right)_{j \in \N},
   \end{split}\label{eq:explicit_dp}
\end{equation}
where $z^{\otimes j}$ denotes the tensor product of order $j$ of the vector
$z$. Technically, the explicit mapping lives in the Hilbert space $\oplus_{j=0}^n \otimes^j \Pcal$,
where $\oplus$ denotes the direct sum of Hilbert spaces, and with the abuse of notation that $\otimes^0 \Pcal$ is simply $\Real$. 
Then, we have that $K_{\text{dp}}(z,z') = \langle \psi(z), \psi(z') \rangle$ for all $z,z'$
   in the unit ball of~$\Pcal$.  Similarly, we can construct an explicit
   feature map for the homogeneous dot-product
   kernels~(\ref{eq:dp_kernel_appx}):
\begin{equation}
   \begin{split}
      \psi_{\text{hdp}}(z) & = \left(\sqrt{b_0}\|z\|, \sqrt{b_1} z, \sqrt{b_2} \|z\|^{-1} z \otimes z,  \sqrt{b_3} \|z\|^{-2} z \otimes z \otimes z, \ldots \right) \\
              & = \left(\sqrt{b_j} \|z\|^{1-j} z^{\otimes j}   \right)_{j \in \N}.
   \end{split}\label{eq:explicit_hdp}
\end{equation}
   From these mappings, we may now conclude the proof by following the same strategy as~\citet{zhang2016l1,zhang2016convexified}. 
By first considering the restriction of $K$ to unit-norm vectors~$z$, 
   \begin{displaymath}
      \sigma(\langle w, z \rangle) = \sum_{j=0}^{+\infty}a_j \langle w, z \rangle^j = \sum_{j=0}^{+\infty}a_j \langle w^{\otimes j}, z^{\otimes j} \rangle = \langle \bar{w}, \psi(z) \rangle,
   \end{displaymath}
   where 
   \begin{displaymath}
      \bar{w} = \left(  \frac{a_j}{\sqrt{b_j}}  w^{\otimes j}  \right)_{j \in \N}.
   \end{displaymath}
   Then, the norm of $\bar{w}$ is 
   \begin{displaymath}
      \|\bar{w}\|^2 = \sum_{j=0}^{+\infty} \frac{a_j^2}{b_j} \| w^{\otimes j} \|^2 = \sum_{j=0}^{+\infty} \frac{a_j^2}{b_j} \| w \|^{2j} = C_{\sigma}^2( \|w\|^2) < +\infty. 
   \end{displaymath}
   Using Theorem~\ref{thm:rkhs}, we conclude that $f$ is in the RKHS of $K$, with norm $\|f\| \leq C_\sigma(\|w\|^2)$.
   Finally, we extend the result to non unit-norm vectors $z$ with similar calculations and we obtain the desired result.
\end{proof}

\subsection{CNN construction and RKHS norm}
\label{sub:cnn_construction_rkhs}

In this section, we describe the space of functions (RKHS) $\Hcal_
{\mathcal{K}_n}$ associated to the kernel~$\mathcal{K}_n(x_0, x_0') = \langle x_n, x_n' \rangle$ defined in~\eqref{eq:linear_prediction_kernel},
where~$x_n$, $x_n'$ are the final representations given by Eq.~\eqref{eq:final_repr},
in particular showing it contains the set of CNNs with
activations described in Section~\ref{sub:rkhs_activations}.

\subsubsection{Construction of a CNN in the RKHS.}
\label{ssub:rkhs_construction}
Let us consider the definition of the CNN presented in Section~\ref{sec:link_with_cnns}. We will show
that it can be seen as a point in the RKHS of~${\mathcal K}_n$.
According to Lemma~\ref{lemma:activations}, we consider~$\Hc_k$ that contains all functions of the form $z \in \Pcal_k \mapsto \|z\| \activ
(\langle w, z \rangle / \|z\|)$, with~$w \in \Pcal_k$. 

We recall the intermediate quantities introduced in Section~\ref{sec:link_with_cnns}.
That is, we define the initial quantities
 $f_1^{i} \in \mathcal{H}_1, g^i_1 \in \Pcal_1$ for~$i = 1, \ldots, p_1$ such that
\begin{align*}
   g_1^i &= w_1^{i}  \in L^2(S_1,\Real^{p_0}) = L^2(S_1,\Hcal_0) = \Pcal_1\\
f_1^i(z) &= \|z\| \activ(\langle g^0_i, z \rangle / \|z\|) \quad \text{ for } z \in \Pcal_1, 
\end{align*}
and we recursively define, from layer $\kmone$, the quantities $f_k^{i} \in \mathcal{H}_k, g_{k}^i \in \Pcal_k$
for $i = 1, \ldots, p_k$:
\begin{align*}
   g_{k}^i(v) &= \sum_{j=1}^{p_{\kmone}} w_{k}^{ij}(v) f_{\kmone}^j ~~~\text{where}~~~ w_k^i(v) = (w_k^{ij}(v))_{j=1,\ldots,p_{\kmone}}\\
f_k^i(z) &= \|z\| \activ(\langle g_{k}^i, z \rangle / \|z\|) \quad \text{ for } z \in \Pcal_k.
\end{align*}

Then, we will show that $\tilde{z}_k^i(u) = f_k^i(P_k x_{\kmone} (u)) = \langle f_k^i, M_k P_k x_{\kmone}(u)\rangle$, which correspond to feature maps at layer~$k$ and index~$i$ in a CNN. Indeed, this is easy to see for~$k = 1$ by construction with filters~$w_1^i(v)$, and for $k \geq 2$, we have
\begin{align*}
   \tilde{z}_k^i(u) &= n_k(u) \sigma\! \left( \langle w_k^i, P_k z_{\kmone}(u) \rangle / n_k(u) \right) \\
    &= n_k(u) \sigma\! \left( \langle w_k^i, P_k A_{\kmone} \tilde{z}_{\kmone}(u) \rangle / n_k(u) \right) \\
        &= n_k(u) \activ \left( \frac{1}{n_k(u)} \sum_{j=1}^{p_{\kmone}} \int_{S_k} w_{k}^{ij}(v)  A_{\kmone} \tilde{z}_{\kmone}^j(u+v)  d \nu_{k}(v) \right)\\
        &= n_k(u) \activ \left( \frac{1}{n_k(u)} \sum_{j=1}^{p_{\kmone}} \int_{S_k} w_{k}^{ij}(v)   \langle f_{\kmone}^j,  A_{\kmone}M_{\kmone} P_{\kmone} x_{\kmtwo}(u+v)\rangle d \nu_{k}(v) \right)\\
        &= n_k(u) \activ \left( \frac{1}{n_k(u)}  \int_{S_k}  \langle g_k^i(v) ,  A_{\kmone} M_{\kmone} P_{\kmone} x_{\kmtwo}(u+v)\rangle d \nu_{k}(v) \right)\\
        &= n_k(u) \activ \left( \frac{1}{n_k(u)}  \int_{S_k}  \langle g_k^i(v) ,  x_{\kmone}(u+v)\rangle d \nu_{k}(v) \right)\\
        &= n_k(u) \activ \left( \frac{1}{n_k(u)}  \langle g_k^i(v) ,  P_k x_{\kmone}(u)\rangle \right) \\
        & = f_k^i(P_k x_{\kmone}(u)),
\end{align*}
where $n_k(u) := \| P_k x_{\kmone} (u) \|$. Note that we have used many times the fact that $A_k$ operates on each channel independently when applied to a finite-dimensional map.

The final prediction function is of the form~$f_\sigma(x_0)= \langle w_{n+1}, z_n \rangle$ with $w_{n+1}$ in $L^2(\Omega, \Real^{p_n})$.
Then, we can define the following function $g_{\sigma}$ in $L^2(\Omega,\Hcal_n)$ such that
\begin{displaymath}
   g_{\sigma}(u) = \sum_{j=1}^{p_n} w_{n+1}^j(u) f_n^j,
\end{displaymath}
which yields
\begin{align*}
   \langle g_{\sigma}, x_n \rangle &= \sum_{j=1}^{p_n} \int_\Omega w_{n+1}^j(u) \langle f_n^j, x_n(u) \rangle du \\
   &= \sum_{j=1}^{p_n} \int_\Omega w_{n+1}^j(u) \langle f_n^j, A_n M_n P_n x_{n-1}(u) \rangle du \\
   &= \sum_{j=1}^{p_n} \int_\Omega w_{n+1}^j(u) A_n\tilde{z}_n^j(u)  du \\
   &= \sum_{j=1}^{p_n} \int_\Omega w_{n+1}^j(u) {z}_n^j(u)  du \\
        &= \sum_{j=1}^{p_n} \langle w_{n+1}^j, z_n^j \rangle = f_{\sigma}(x_0),
\end{align*}
which corresponds to a linear layer after pooling. Since the RKHS of ${\mathcal K}_n$ in the linear case~\eqref{eq:linear_prediction_kernel} contains all functions of the form $f(x_0) = \langle g, x_n \rangle$, for $g$ in $L^2(\Omega,\Hcal_n)$, we have that $f_\sigma$ is in the RKHS.

\subsubsection{Proof of Proposition~\ref{prop:cnns}}
\label{ssub:cnn_norm_proof}
\begin{proof}
As shown in Lemma~\ref{lemma:activations}, the RKHS norm of a function $f: z \in \Pcal_k \mapsto \|z\| \activ(\langle w, z \rangle/\|z\|)$ in~$\Hc_k$ is bounded by $C_\activ(\|w\|^2)$, where~$C_\activ$ depends on the activation~$\activ$. We then have
\begin{align*}
\|f_1^i\|^2 &\leq C_\activ^2(\|w_1^i\|_2^2 ) \quad \text{where} \quad \|w_1^i\|_2^2 = \int_{S_1}\|w_1^i(v)\|^2 d \nu_1(v) \\
\|f_{k}^i\|^2 &\leq C_\activ^2(\|g_k^i\|^2) \\
   \|g_k^i\|^2 &= \int_{S_{k}} \|\sum_{j=1}^{p_{\kmone}} w_k^{ij}(v) f_{\kmone}^j \|^2 d \nu_{k}(v) \\
        &\leq p_{\kmone} \sum_{j=1}^{p_{\kmone}} \left(\int_{S_k} |w_k^{ij}(v)|^2 d \nu_{k}(v) \right) \|f_{\kmone}^j\|^2 \\
        &= p_{\kmone} \sum_{j=1}^{p_{\kmone}} \|w_k^{ij}\|^2_2 \|f_{\kmone}^j\|^2,
\end{align*}
where in the last inequality we use $\|a_1 + \ldots + a_n\|^2 \leq n (\|a_1\|^2 + \ldots + \|a_n\|^2)$.
Since $C_\activ^2$ is monotonically increasing (typically exponentially in its argument),
we have for $k = 1, \ldots, n-1$ the recursive relation
\begin{equation*}
   \|f_{k}^i\|^2 \leq C_\activ^2 \left(p_{\kmone} \sum_{j=1}^{p_{\kmone}} \|w_k^{ij}\|^2_2 \|f_{\kmone}^j\|^2 \right).
\end{equation*}
The norm of the final prediction function~$f \in L^2(\Omega, \Hc_n)$ is bounded
as follows, using similar arguments as well as Theorem~\ref{thm:rkhs}:
\begin{align*}
   \|f_{\sigma}\|^2 \leq \|g_\sigma\|^2 \leq p_n \sum_{j=1}^{p_n} \left(\int_\Omega |w_{n+1}^j(u)|^2 du\right) \|f_n^j\|^2.
\end{align*}
This yields the desired result.
\end{proof}

\subsubsection{Proof of Proposition~\ref{prop:cnns_spectral}}
\label{ssub:cnn_spectral_norm_proof}
\begin{proof}
Define
\begin{align*}
F_k &= (f_k^1, \ldots, f_k^{p_k}) \in \Hc_k^{p_k} \\
G_k &= (g_k^1, \ldots, g_k^{p_k}) \in \Pcal_k^{p_k} \\
W_k(u) &= (w_k^{ij}(u))_{ij} \in \R^{p_k \times p_{\kmone}} \quad \text{ for }u \in S_k.
\end{align*}
We will write, by abuse of notation,~$G_k(u) = (g_k^1(u), \ldots, g_k^{p_k}(u))$ for~$u \in S_k$, so that we can write $G_k(u) = W_k(u) F_{\kmone}$.
In particular, we have~$\|G_k(u)\| \leq \|W_k(u)\|_2 \|F_{\kmone}\|$.
This can be seen by considering an orthonormal basis~$B$ of~$\Hc_k$,
and defining real-valued vectors $F_k^w = (\langle w, f_k^1 \rangle, \ldots, \langle w, f_k^{p_k} \rangle)$,
$G_k^w(u) = (\langle w, g_k^1(u) \rangle, \ldots, \langle w, g_k^{p_k}(u) \rangle)$ for~$w \in B$.
Indeed, we have $G_k^w(u) = W_k(u) F_{\kmone}^w$ and hence $\|G_k^w(u)\| \leq \|W_k(u)\|_2 \|F^w_{\kmone}\|$ for all~$w\in B$, and we conclude using
\begin{align*}
\|G_k(u)\|^2 = \sum_{w\in B} \|G_k^w(u)\|^2 \leq \|W_k(u)\|_2^2 \sum_{w\in B} \|F^w_{\kmone}\|^2 = \|W_k(u)\|_2^2 ~ \|F_{\kmone}\|^2.
\end{align*}
Then, we have
\begin{align*}
\|G_k\|^2 &= \sum_i \|g_k^i\|^2 = \sum_i \int_{S_k} \|g_k^i(u)\|^2 d \nu_k(u) = \int_{S_k} \|G_k(u)\|^2 d \nu_k(u) \\
  &\leq \int_{S_k} \|W_k(u)\|_2^2 ~\|F_{\kmone}\|^2 \nu_k(u) = \|W_k\|_2^2 ~\|F_{\kmone}\|^2.
\end{align*}
Separately, we notice that $C_\sigma^2$ is super-additive, \ie,
\begin{align*}
C_\sigma^2(\lambda_1^2 + \ldots + \lambda_n^2) \geq C_\sigma^2(\lambda_1^2) + \ldots + C_\sigma^2(\lambda_n^2).
\end{align*}
Indeed, this follows from the definition of~$C_\sigma^2$, noting that polynomials with non-negative coefficients are super-additive on non-negative numbers.
Thus, we have
\begin{align*}
\|F_1\|^2 &= \sum_{i=1}^{p_1} \|f_1^i\|^2 \leq \sum_{i=1}^{p_1} C_\sigma^2(\|w_1^i\|^2) \leq C_\sigma^2(\|W_1\|_F^2) \\
\|F_k\|^2 &\leq \sum_{i=1}^{p_k} C_\sigma^2(\|g_k^i\|^2) \leq C_\sigma^2(\|G_k\|^2), \quad \text{ for }k=2, \ldots, n.
\end{align*}
Finally, note that
\[
\|g_\sigma(u)\|^2 \leq \left( \sum_{j=1}^{p_n} |w_{n+1}^j(u)| \|f_n^j\|\right)^2 \leq \|w_{n+1}(u)\|^2 \|F_n\|^2,
\]
by using Cauchy-Schwarz, so that $\|g_\sigma\|^2 \leq \|w_{n+1}\|^2 \|F_n\|^2$.
Thus, combining the previous relations yields
\begin{align*}
\|f_{\sigma}\|^2 &\leq \|g_\sigma\|^2 \leq \|w_{n+1}\|^2 ~ C_\sigma^2(\|W_n\|_2^2 ~ C_\sigma^2(\|W_{n-1}\|_2^2 \ldots C_\sigma^2(\|W_1\|_F^2) \ldots)),
\end{align*}
which is the desired result.
\end{proof}


\vskip 0.2in
\bibliography{full,bibli}

\end{document}